\documentclass[11pt]{article} 

 \usepackage[utf8]{inputenc} 
 \usepackage[T1]{fontenc}    
 \usepackage{url}            
 \usepackage{booktabs}       
 \usepackage{amsfonts}       
 \usepackage{nicefrac}       
 \usepackage{microtype}      

\usepackage{fncylab,enumerate,wrapfig,placeins}
\usepackage{bbm}
\usepackage{graphicx}
\usepackage{amsmath,amssymb}  
\usepackage{amsmath,amsthm,amsfonts,amssymb}  

\usepackage{geometry}

\pdfoutput=1  

\usepackage{subcaption}
\usepackage{booktabs}

\usepackage{multirow}

\usepackage{algorithm}
\usepackage{algpseudocode}

\usepackage[usenames,dvipsnames]{color}
\usepackage{color}

\usepackage[colorlinks,%
linkcolor=BrickRed,%
filecolor =BrickRed,%
urlcolor=blue,%
citecolor=RoyalPurple,%
]{hyperref}

\graphicspath{{/}}

\newlength{\noteWidth}
\long\def\notes#1{\ifinner
{\footnotesize #1}
\else 
\marginpar{\parbox[t]{\noteWidth}{\raggedright\footnotesize#1}}
\fi\typeout{#1}}

\def\notes#1{}

\def\mindex#1{\index{#1}}

\def\sq{\hbox{\rlap{$\sqcap$}$\sqcup$}}
\def\qed{\ifmmode\sq\else{\unskip\nobreak\hfil
\penalty50\hskip1em\null\nobreak\hfil\sq
\parfillskip=0pt\finalhyphendemerits=0\endgraf}\fi\medskip}

\long\def\defbox#1{\framebox[.9\hsize][c]{\parbox{.85\hsize}{%
\parindent=0pt
\baselineskip=12pt plus .1pt      
\parskip=6pt plus 1.5pt minus 1pt 
 #1}}}

\long\def\beginbox#1\endbox{\subsection*{}%
\hbox{\hspace{.05\hsize}\defbox{\medskip#1\bigskip}}%
\subsection*{}}

\def\endbox{}

\def\transpose{{\intercal}}

\newsavebox{\junk}
\savebox{\junk}[1.6mm]{\hbox{$|\!|\!|$}}

\def\limsup{\mathop{\rm lim\ sup}}
\def\liminf{\mathop{\rm lim\ inf}}

\def\argmax{\mathop{\rm arg\, max}}

\def\U{{\sf U}}

\def\state{{\sf X}}

\def\zstate{{\sf Z}}

\newcommand{\field}[1]{\mathbb{#1}}

\def\posRe{\field{R}_+}
\def\Re{\field{R}}

\def\ind{\field{I}}

\def\intgr{\field{Z}}

\def\Co{\field{C}}

\def\bfmath#1{{\mathchoice{\mbox{\boldmath$#1$}}%
{\mbox{\boldmath$#1$}}%
{\mbox{\boldmath$\scriptstyle#1$}}%
{\mbox{\boldmath$\scriptscriptstyle#1$}}}}

\def\bfmU{\bfmath{U}}

\def\bfmX{\bfmath{X}}

\def\bfmY{\bfmath{Y}}

\def\bfmhhaY{\bfmath{\hhaY}} 
\def\bfmhhaY{\hbox to 0pt{$\widehat{\bfmY}$\hss}\widehat{\phantom{\raise 1.25pt\hbox{$\bfmY$}}}}

\def\bfPhi{\bfmath{\Phi}}

\def\til={{\widetilde =}}

\def\tiltheta{\widetilde \theta}
\def\tilSigma{\widetilde{\Sigma}}
\def\tilG{\widetilde{G}}

\def\tiltheta{{\tilde \theta}}

\def\clA{{\cal A}}
\def\clB{{\cal B}}

\def\clD{{\cal D}}
\def\clE{{\cal E}}

\def\clG{{\cal G}}
\def\clH{{\cal H}}

\def\clN{{\cal N}}

\def\clR{{\cal R}}

\def\clU{{\cal U}}

 \def\FRAC#1#2#3{\genfrac{}{}{}{#1}{#2}{#3}}

\def\fraction#1#2{{\mathchoice{\FRAC{0}{#1}{#2}}%
{\FRAC{1}{#1}{#2}}%
{\FRAC{3}{#1}{#2}}%
{\FRAC{3}{#1}{#2}}}}

\def\ddt{{\mathchoice{\FRAC{1}{d}{dt}}%
{\FRAC{1}{d}{dt}}%
{\FRAC{3}{d}{dt}}%
{\FRAC{3}{d}{dt}}}}

\def\ddtp{{\mathchoice{\FRAC{1}{d^{\hbox to 2pt{\rm\tiny +\hss}}}{dt}}%
{\FRAC{1}{d^{\hbox to 2pt{\rm\tiny +\hss}}}{dt}}%
{\FRAC{3}{d^{\hbox to 2pt{\rm\tiny +\hss}}}{dt}}%
{\FRAC{3}{d^{\hbox to 2pt{\rm\tiny +\hss}}}{dt}}}}

\def\half{{\mathchoice{\FRAC{1}{1}{2}}%
{\FRAC{1}{1}{2}}%
{\FRAC{3}{1}{2}}%
{\FRAC{3}{1}{2}}}}

\def\eqdef{\mathbin{:=}}

\def\Expect{{\sf E}}

\def\average#1,#2,{{1\over #2} \sum_{#1}^{#2}}

\def\eye(#1){{\bf(#1)}\quad}
\def\epsy{\varepsilon}

\newtheorem{theorem}{Theorem}[section]

\newtheorem{proposition}[theorem]{Proposition}
\newtheorem{lemma}[theorem]{Lemma}

\def\Lemma#1{Lemma~\ref{#1}}
\def\Proposition#1{Prop.~\ref{#1}}
\def\Theorem#1{Theorem~\ref{#1}}

\def\Section#1{Section~\ref{#1}}

\def\Figure#1{Figure~\ref{#1}}

\def\eq#1/{(\ref{e:#1})}

\newcommand{\beqn}[1]{\notes{#1}%
\begin{eqnarray} \elabel{#1}}

\newcommand{\eeqn}{\end{eqnarray} }

\newcommand{\beq}[1]{\notes{#1}%
\begin{equation}\elabel{#1}}

\newcommand{\eeq}{\end{equation}}

\def\bdes{\begin{description}}
\def\edes{\end{description}}

\def\barc{{\overline {c}}}

\def\barf{{\overline {f}}}

\newcounter{rmnum}
\newenvironment{romannum}{\begin{list}{{\upshape (\roman{rmnum})}}{\usecounter{rmnum}
\setlength{\leftmargin}{14pt}
\setlength{\rightmargin}{8pt}
\setlength{\itemsep}{2pt}
\setlength{\itemindent}{-1pt}
}}{\end{list}}

\newcounter{anum}
\newenvironment{alphanum}{\begin{list}{{\upshape (\alph{anum})}}{\usecounter{anum}
\setlength{\leftmargin}{14pt}
\setlength{\rightmargin}{8pt}
\setlength{\itemsep}{2pt}
\setlength{\itemindent}{-1pt}
}}{\end{list}}

{\end{list}}

\def\ass(#1:#2){(#1\ref{#1:#2})}

\def\ritem#1{
\item[{\sf \ass(\current_model:#1)}]
}

\newenvironment{recall-ass}[1]{%
\begin{description}
\def\current_model{#1}}{
\end{description}
}

\newcommand{\bd}{\begin{description}}
\newcommand{\ed}{\end{description}}
\newcommand{\bt}{\begin{theorem}}
\newcommand{\et}{\end{theorem}}
\newcommand{\ba}{\begin{array}{rcl}}
\newcommand{\ea}{\end{array}}

\renewenvironment{romannum}{\begin{list}{{\upshape (\roman{rmnum})}}{\usecounter{rmnum}
\setlength{\leftmargin}{14pt}
\setlength{\rightmargin}{8pt}  
\setlength{\itemsep}{2pt}
\setlength{\itemindent}{3pt}
}}{\end{list}}

\def\Real{\text{\rm Re}\,}

\def\termState{{\cal S}}

\def\disc{\gamma}
\def\stepf{\beta}

\def\trans{\intercal}
\def\Alg#1{Alg.~\ref{a:Zapalg}}

\def\Lemma#1{Lemma~\ref{#1}}
\def\Proposition#1{Prop.~\ref{#1}}
\def\Prop#1{Prop.~\ref{#1}}
\def\Theorem#1{Thm.~\ref{#1}}   

\def\state{{\sf X}}
\def\action{{\sf U}}

\def\eqdef{\mathbin{:=}}

\def\elig{\zeta}
\def\bfelig{\bfmath{\zeta}}

\def\uQ{\underline{Q}}

\def\haA{\widehat A}
\def\haAp{\widehat A^{+}}
\def\haAn{\widehat A^{-}}
\def\haAz{\widehat A^{\zeta}}

\def\clA{\mathcal{A}}

\def\barclA{\bar{\cal A}}
\def\barclG{\bar{\cal G}}
\def\barclAp{\bar{\cal A}^+}
\def\barclAn{\bar{\cal A}^-}
\def\barclAz{\bar{\cal A}^{\zeta}} 
\def\clH{{\cal H}}

\def\ind{\field{I}}

\def\posRe{\field{R}_+}
\def\Re{\field{R}}

\def\barGamma{\overline{\Gamma}}

\def\barc{\bar c}

\def\Xt{w}
\def\barxt{\bar \Xt}
\def\dist{\text{dist}}
\def\distn{\text{dist}_{\cal N}}

\def\so{{\scriptstyle {\cal O}}}
\def\unit{\mathbf{1}}
\def\condr{\varsigma}

\title{Zap Q-learning with 
\\
Nonlinear Function Approximation}

\author{
	Shuhang Chen\thanks{Department of Mathematics at the University of Florida, Gainesville.}%
	\and
	Adithya M. Devraj\thanks{Department of ECE at the University of Florida, Gainesville.}%
	\and 
	Fan Lu\footnotemark[2]%
	\and
	Ana Bu\v{s}i\'{c}\thanks{Inria and DI ENS, \'Ecole Normale	Sup\'erieure, CNRS, PSL Research University, Paris, France.      \newline Financial support from ARO grant W911NF1810334 is gratefully acknowledged.  Additional support from EPCN 1935389 \&\ CPS~1646229, and French National Research Agency grant ANR-16-CE05-0008.}%
	\and 
	Sean  Meyn\footnotemark[2]
}

\begin{document}

\maketitle

\begin{abstract}
Zap Q-learning is a recent class of reinforcement learning algorithms, motivated primarily as a means to accelerate convergence.  Stability theory has been absent outside of two restrictive classes: the tabular setting, and optimal stopping.  This paper introduces a new framework for analysis of a more general class of recursive algorithms known as stochastic approximation. Based on this general theory, it is shown that Zap Q-learning is consistent under a non-degeneracy assumption, even when the function approximation architecture is nonlinear. Zap Q-learning with neural network function approximation emerges as a special case, and is tested on examples from OpenAI Gym.  Based on multiple experiments with a range of neural network sizes, it is found that the new algorithms converge quickly and are robust to choice of function approximation architecture.


	%
	%
	%
	%
	%

\end{abstract}


%
%
%
%
%
%
%
%

\section{Introduction}
\label{sec:intro}

A primary goal of 
reinforcement learning (RL)
is the creation of algorithms 
that are convergent, 
converge at the fastest possible rate,
and result in a policy for control that has near optimal performance. 
This paper focuses on algorithm design to ensure stability of the algorithm, consistency, and techniques to obtain qualitative insight on the rate of convergence. One framework for algorithm design is the theory of stochastic approximation (SA).  The main contribution of this work is a new class of Q-learning algorithms that are convergent even for nonlinear function approximation architectures, such as neural networks.

Consider a Markov decision process (MDP) model with state-input sequence $\{ (X_n,U_n) : n\ge 0\}$,  and let $\{ Q^\theta(x,u) : \theta\in\Re^d \}$ denote a family of approximations of the Q-function;
the vector $\theta\in\Re^d$ might correspond to weights in a neural network. 
One popular formulation of Q-learning is defined by the recursion,
\begin{equation}
\theta_{n+1} = \theta_n + \alpha_{n+1} \clD_{n+1}  \elig_n
\label{e:Q0}
\end{equation}
in which  $\{\alpha_{n+1}\} $  is the non-negative step-size sequence, $\{\clD_{n+1} \}$  is the scalar sequence of \textit{temporal differences} [recalled in eq.~\eqref{e:safn-d}],  and   $\{\elig_n\} $ the
\textit{eligibility vectors}:  a typical choice is
\begin{equation}
\elig_n = \nabla_\theta Q^\theta(X_n,U_n) \Big|_{\theta =\theta_n}
\label{e:elig}
\end{equation}
The Q-learning algorithm of Watkins can be expressed as \eqref{e:Q0}, with a linear function approximation $Q^\theta =\theta^\transpose \psi$, 
and the basis functions $\psi(x,u)$ being indicator functions of each state-input pair.   

The theory of Q-learning with function approximation has not caught up with the famous success stories in applications.    
Consistency of the Q-learning algorithm in the tabular setting was established in the seminal work of 
Watkins and Dayan \cite{watday92a}.   Counter-examples   soon followed:   
the recursion \eqref{e:Q0} may fail to converge,  even in the   linear function approximation setting   \cite{bai95, maeszebhasut10}.   
Moreover, even when convergence holds,   Q-learning can be extremely slow \cite{sze97,eveman03,devmey17b}.  

The so-called \textit{ODE method} of SA theory is typically regarded as a method of analysis for stochastic recursions.  We take the opposite view,  regarding an ODE as a first step in algorithm design.  This is motivated in part by the recent work \cite{suboycan14,shidusujor19} concerning the value of high-resolution approximations of ODEs   for applications in optimization,  and the enormous insight gained from a careful inspection of candidate   ODEs.

The Q-learning algorithms considered in the present work are designed to solve a root-finding problem of the form $\barf(\theta^*)\eqdef \Expect[\elig_n\clD_{n+1}] \big|_{\theta=\theta^*}=0$.   For ODE design,  we let $\Xt_t\in\Re^d$ denote the state of the ODE at time $t$,  and seek a vector field $\nu\colon \Re^{d+1}\to\Re^d$ to define the evolution:
\[
\ddt \Xt_t =\nu(\Xt_t,t) 
\]
The vector field is designed so that $\Xt_t\to\theta^*$ from each initial condition, and so that the ODE solutions can be efficiently approximated using a discrete-time algorithm driven by observations.

One approach is to apply gradient descent to solve the  non-convex optimization problem:
\begin{equation}
\label{e:gq-obj}
\min_\theta J(\theta)=\min_{\theta}\half \barf(\theta)^\intercal M\barf(\theta),\qquad \mbox{with} \quad M>0
\end{equation}
which results in the ODE with  time homogeneous vector field:
\begin{equation}
\ddt \Xt_t =  -   [\partial_\theta  \barf\, (\Xt_t)]^\transpose M\barf(\Xt_t)
\label{e:GQODE}
\end{equation} 
The GQ-learning algorithm of \cite{maeszebhasut10} can be regarded as a direct discrete-time translation of this ODE,  using $M = \Expect[\elig_n\elig_n^\transpose]^{-1}$.

This approach is discussed in Nesterov's monograph \cite[Section 4.4.1]{nes18} for general root finding problems,  who warns that it can lead to numerical instability: ``...if our system of equations is linear, then such a transformation squares the condition number of the problem''.   
He goes on to warn that it can lead to a ``squaring the number of iterations'' to obtain the desired error bound.   

%
%
%

The main results of the present paper are related to
the  \textit{Newton-Raphson flow},  defined by another time homogeneous vector field $\nu(\Xt) =  - [\partial_\theta  \barf\, (\Xt)]^{-1} \barf(\Xt)$:
\begin{equation}
\ddt \Xt_t = G_t  \barf(\Xt_t)\,,\qquad \textit{with}\quad  G_t = -[\partial_\theta  \barf\, (\Xt_t)]^{-1}
\label{e:ZapODE1}
\end{equation}
A change of variables
leads to the  linear dynamics,
$\ddt \barf(\Xt_t)  = - \barf(\Xt_t) $,  with solution
\begin{equation}
\barf(\Xt_t) =  \barf(\Xt_0) \, e^{-t}\,, \qquad t\ge 0
\label{e:ZapODE2}
\end{equation}
Thus,  provided solutions to \eqref{e:ZapODE1} are bounded,
the algorithm is consistent
in the sense that the limit points of $\{\Xt_t \}$ lie in the set of roots $\Theta^* \eqdef  \{ \theta : \barf(\theta) =0 \}$.

In most applications it is not possible to determine a-priori if the matrix $\partial_\theta  \barf\,(\theta)$ is   full rank, which motivates a \emph{regularized  Newton-Raphson flow}:   
\begin{equation}
\ddt \Xt_t =    - [\epsy I +A(\Xt_t) ^\intercal A(\Xt_t) ]^{-1}  A(\Xt_t) ^\intercal     \barf(\Xt_t)  \,,   \qquad  A(\Xt_t)=  \partial_\theta \barf\,  ( \Xt_t)
\label{e:ZAPODE}
\end{equation} 
It is shown in \Prop{t:odestable} that \eqref{e:ZAPODE} is stable, provided    $V = \|\barf\|^2$  is a coercive function on $\Re^d$; $V$ serves as a Lyapunov function for  \eqref{e:ZAPODE},   giving 
\begin{equation} 
\lim_{t\to\infty}  \barf(\Xt_t)   = 0
\label{e:QAf0}
\end{equation}
Hence the limit points of solutions lie in the set $\Theta^*  $.



Details of the algorithms and the  contributions of this paper require additional background.   
Consider the  
$d$-dimensional SA recursion of Robbins and Monro \cite{robmon51a,bor08a}:
\begin{equation}
\theta_{n+1} = 
\theta_n + \alpha_{n+1} f(\theta_n, \Phi_{n+1}) 
\label{e:SA}
\end{equation}
in which $\bfPhi$ is an irreducible Markov chain on a finite state space $\zstate$,     $\{\alpha_n\}$ is a     non-negative gain sequence,
and   $f\colon\Re^d\times \zstate\to\Re^d$.  
It is assumed that $\bfPhi$ has a unique invariant probability mass function (pmf) on $\zstate$.  
The algorithm is designed to approximate roots of the function  $\barf(\theta) =\Expect[ f(\theta, \Phi_{n+1}) ]$  (with expectation in steady-state).   
Under mild conditions, the  SA recursion shares the same limit points as the ODE $
\ddt\Xt_t \! = \! \barf(\Xt_t)
$  \cite{kusyin97,bor08a,benmetpri12}.  More recently it has been established that boundedness of the stochastic recursion follows from a stability condition for the ODE \cite{bormey00a,bor08a,rambha17}  (prior to this work,  stability of the stochastic recursion required separate arguments  \cite{tsi94a}).  

%


One approach to obtain a rate of convergence in SA is   through the linearization:
\begin{equation}
\clE_{n+1}  = \clE_n + \alpha_{n+1} [A_*  \clE_n  + \Delta_{n+1}  ]  \,, \qquad {\clE_0 = \theta_0 -\theta^*}
\label{e:SAlinearized}
\end{equation}
where   $A_*  = \partial_\theta \barf\,  ( \theta^*)$ is called the \emph{linearization matrix}.
The sequence   $\Delta_{n+1} \eqdef f(\theta^*, \Phi_{n+1}) $ is assumed to admit a Central Limit Theorem (CLT) in the usual sense, with asymptotic covariance
\begin{equation}
\Sigma_\Delta  = \sum_{k=-\infty}^\infty  \Expect[\Delta_k\Delta_0^\transpose]
\label{e:SigmaDelta}
\end{equation}
where the expectations are in steady state.   The approximation $\clE_n \approx  \tiltheta_n \eqdef \theta_n - \theta^*$ holds
under additional stability assumptions on the stochastic recursion \eqref{e:SA}, which in particular leads
to a CLT for the scaled error $\sqrt{n} \tiltheta_n$ \cite{kusyin97,bor08a,benmetpri12}. 

The asymptotic covariance $\Sigma_{\theta}$ in the CLT  has a simple form, subject to the eigenvalue test:
\begin{equation}
\Real(\lambda)<-\half \qquad \textit{for each eigenvalue $\lambda$ of $A_*$}
\label{e:eigTest}
\end{equation}   
Under this assumption,
$\Sigma_{\theta}$ is the unique solution of the   Lyapunov equation,
\begin{equation}
[\half I +A_*]\Sigma_\theta + \Sigma_\theta[\half I + A_*]^\transpose + \Sigma_{\Delta} = 0
\label{e:eigTestLyap}
\end{equation}  
For a fixed but arbitrary initial condition   $(\Phi_0,\clE_0)$, denote
$\Sigma_n = \Expect[\clE_n\clE_n^\transpose]$. 
The following bounds were obtained in \cite{chedevbusmey20a} for the linear recursion \eqref{e:SAlinearized}:  
\begin{romannum}
\item  If \eqref{e:eigTest} holds, then
$\Sigma_n= n^{-1}\Sigma_\theta +O(n^{-1-\delta})$ for some $\delta >0$.


\item If there exists an eigenvalue of $A_*$ with $\rho \eqdef -\Real(\lambda) < \half$,  and associated eigenvector $v$ satisfying $\Sigma_\Delta v\neq 0$, then
the convergence rate of $\Sigma_n $ to zero is no faster than $n^{-2\rho} $.
\end{romannum}


 Even though the recursion for Watkins' Q-learning is of the form \eqref{e:Q0}, with $\clD_{n+1}$ a \emph{non-linear} function of $\theta_n$, techniques of \cite{sze97} can be used to show that the estimates obtained using the non-linear recursion \emph{couple} with the estimates of a linear recursion of the form \eqref{e:SAlinearized} (see \cite{devmey20a} for a discussion).
The slow convergence for Watkins' algorithm can then be explained by the fact that we are in case (ii) for the linearized recursion,  whenever the discount factor satisfies $\disc>\half$:   for a standard step-size rule, the maximal eigenvalue of $A_*$ in Watkins' Q-learning is $\lambda=-(1-\disc)$,  and the condition $\Sigma_\Delta v\neq 0$ holds under very mild conditions on the MDP  \cite{devmey17b}.  It follows that the mean square error converges to zero at rate   $n^{-2 (1-\disc)}$.   
 For GQ-learning, it is shown in Appendix~\ref{sec:analysis-gq} that the maximal eigenvalue is greater than $ - (1-\disc)^2$, which is consistent with Nesterov's warning.



The slow convergence can be remedied by scaling the step-size by a constant $g>1$ (sufficiently large so that the matrix $\half I+ gA_* $ is Hurwitz).    
For tabular Q-learning any value satisfying $g> 1/[2(1-\disc)]$ will suffice,  while for GQ-learning the scaling must be increased beyond $1/[2(1-\disc)^2]$. 
  Unfortunately, this approach may lead to very high variance.   



%


\paragraph{Contributions}
\textbf{(i)}   A generalization of the Zap SA algorithm of \cite{devmey17b, devmey20a} is proposed.
\\
{\bf{\emph{Zap SA Algorithm:}}} Initialize $\theta_0 \in \Re^d$, $\haA_0 \in \Re^{d \times d}$, $\epsy > 0$. Update for $n \geq 0$:
\vspace{-0.02in}
\begin{subequations}
\begin{flalign}
 \haA_{n+1} & = \haA_n + \stepf_{n+1}  \bigl[  A_{n+1} (\theta_n) - \haA_n  \bigr],   \quad  &A_{n+1}(\theta) & \eqdef  \partial_\theta  f \, (\theta,\Phi_{n+1})
\\
\theta_{n+1} & = \theta_n + \alpha_{n+1}G_{n+1} f(\theta_n, \Phi_{n+1}),   \quad    & G_{n+1} & \eqdef  -  [\epsy I + \haA_{n+1}^\intercal \haA_{n+1}]^{-1}  \haA_{n+1}^\intercal 
\end{flalign}
\label{e:ZapSAalgo}
\end{subequations}
The algorithm is designed so that it approximates the ODE \eqref{e:ZAPODE},  which requires  $\alpha_n = o(\stepf_n)$.

\textbf{(ii)}   A special case of this new class of SA algorithms leads to a significant generalization of   \textit{Zap Q-learning}, for which convergence theory is obtained even in  a \emph{nonlinear} function approximation setting.    The reliability in neural network function approximation architectures is tested through simulations.
	

\textbf{(iii)}   The main technical contribution of this paper is an extension of SA theory to Zap Q-learning, and as a byproduct also GQ-learning,  by exploiting approximate convexity/concavity of the functions $f$ and $\barf$ defined implicitly in  \eqref{e:ZapSAalgo}.


Contribution~(iii) resolves a significant challenge for both Zap Q-learning and GQ-learning:   the \textit{approximation} in stochastic approximation.  
Standard theory does not apply  because   $A(\theta) \eqdef \partial_\theta  \barf(\theta)$  is not continuous.   
An ODE approximation for GQ-learning is obtained in \cite{maeszebhasut10} through the assumption that noise $\{\Delta_n\}$ defined below \eqref{e:SAlinearized} is martingale-difference.  Assumption (Q3) of  \cite{devmey17b} is introduced to obtain an ODE approximation without this restrictive assumption on noise.  However,  this     implicit assumption   cannot be tested a-priori.


\paragraph{Literature review}


The observation that many RL algorithms can be cast as  SA first appeared in \cite{tsi94a,jaajorsin94a}.
Soon after, SA theory was applied to obtain stability theory for TD-learning with linear function approximation  under minimal assumptions  \cite{tsiroy97a}; the authors   discussed challenges for nonlinear approximation architectures.  



In the case of Q-learning,    ODE approximations are nonlinear and not understood outside of a few special cases (notably tabular, and optimal stopping with linear function approximation).   There are many counterexamples showing that conditions on the function class are required in general,  even in a linear function approximation setting  \cite{baird1995residual} (also see \cite{tsiroy94, sut95, gor00}).   
There has also been progress for general linear function approximation: 
sufficient conditions for convergence of the basic Q-learning algorithm \eqref{e:Q0}  was obtained in \cite{melmeyrib08}, with finite-$n$ bounds appearing recently in \cite{chezhadoa_19}, and stability of GQ-learning  was established in  \cite{maeszebhasut10}  subject to assumptions slightly stronger than (A1)--(A3) in the present paper.  In particular, it
is assumed in \cite[Assumption~L3]{maeszebhasut10} that $\partial_\theta  \barf\, (\theta)$ is everywhere nonsingular. In \cite{chizhuzhamic19}, the authors obtained regret bounds for Q-learning in an episodic setting, under a \emph{linear MDP} (linear dynamics and linear rewards) assumption, stronger than the assumptions imposed here.


Stability theory for off-policy TD-learning faces similar challenges as Q-learning.   A consistent algorithm is introduced in  \cite{sutszemae08} for linear function approximation, using the same ideas as in  \cite{maeszebhasut10}; this theory is extended to non-linear function approximation in  \cite{shadoidav09}.

%


 To the best of our knowledge, the ODE  \eqref{e:ZapODE1} was introduced in the economics literature, which led to the comprehensive analysis  by Smale~\cite{sma76}   for   smooth $\barf $.
The term \textit{Newton-Raphson flow} for \eqref{e:ZapODE1} was introduced in the deterministic control literature \cite{shibucwarseaege18,warseaegebuc17}.  
The Zap SA algorithm   was introduced at the same time, and based on the same ODE~\cite{devmey17a}.

The motivation of \cite{devmey17a} was centered entirely on optimizing the asymptotic covariance of stochastic approximation, and in particular Q-learning with tabular basis;  see  \cite{kusyin97,benmetpri12} for history of convergence rate theory in SA,  and \cite{kon02,kontsi04} for application to actor-critic methods.
While the motivation here is stability,   results in \Section{s:var_approx_proof} strongly suggest that the asymptotic covariance is approximately optimal for the regularized Zap Q-learning algorithm  introduced here;   a   ``tightness argument'' is required to complete the proof.

The analysis   in this paper can be cast in the general framework of stochastic approximation based on differential inclusions (see \cite[Chapter 5]{bor08a} and its references).    This general framework guided the research reported here.  New in this paper is the proof of convergence of Zap~Q-learning via an ODE approximation,  made possible by the special structure of   the recursion.

\section{Zap~Q-learning with Nonlinear Function Approximation}
\label{sec:zap-q-learning}

%
\paragraph{Preliminaries}

We restrict to a discounted reward optimal control problem, with finite state space $\state$, finite input space $\action$, reward function $r:\state\times \action\rightarrow \Re$, and discount factor $\disc\in (0, 1)$.   
The Q-function is defined as the maximum over all possible input sequences $\{U_n : n \geq 1\}$ of the total discounted reward: 
\begin{equation}
Q^*(x,u)\eqdef \max_\bfmU\sum_{n=0}^\infty \disc^n\Expect[r(X_n,U_n)\mid X_0=x \,, U_0=u]\,,\qquad x\in\state, u\in\action
\label{e:Q}
\end{equation}
Extensions to other criteria are straightforward (e.g.,    average cost or weighted shortest path).

Let $P_u$ denote the state transition matrix when input $u\in \action$ is taken. It is known that the Q-function is the unique solution to the Bellman equation \cite{bertsi96a}:
\begin{equation}
\label{e:qfunction}
Q^*(x,u) = r(x,u) + \disc\sum_{x'\in\state}P_u(x,x')\uQ^*(x')        
\end{equation}
where $\uQ(x) \eqdef\max_{u\in\action}Q(x,u)$ for any function $Q:\state\times\action\rightarrow \Re$. 


Consider a (possibly nonlinear) parameterized family of candidate approximations $\{Q^\theta:\theta\in \Re^d\}$, wherein $Q^\theta:\state\times\action\rightarrow \Re$ for each $\theta$,   and  the associated family of policies $ \phi^\theta(x) \in\argmax_{u} Q^\theta(x,u)$,  $x\in\state$.   To avoid ambiguities when the maximizer is not unique, we enumerate all stationary policies as $\{\phi^{(i)}:1\leq i\leq\ell_\phi\}$,  and specify 
 \begin{equation}
 \label{e:phi}
 \phi^\theta \eqdef \phi^{(\kappa)}\,, \qquad \text{ where } \quad \kappa \eqdef \min\{i:\phi^{(i)} (x) \in\argmax_u Q^\theta(x,u), \mbox{ for all } x\in\state \} 
 \end{equation}
 The recursion \eqref{e:Q0} is designed to compute  an approximate solution of \eqref{e:qfunction},  defined as the solution to the root-finding problem:
\begin{equation}
\label{e:galerkin}
\barf(\theta^*) = 0,\qquad \mbox{with }\quad\barf(\theta)\eqdef\Expect \big [ \big(r(X_n, U_n) + \disc\uQ^{\theta}(X_{n+1}) - Q^{\theta}(X_n,U_n) \big) \elig_n\big]
\end{equation}
where  the expectation is in steady state.

It is convenient to denote $\Phi_{n+1}\eqdef (X_{n+1}, X_n,U_{n+1}, U_n)$, with state space $\zstate\eqdef \state^2\times \action^2 $.
It is assumed throughout the paper that $\elig_n = \zeta(\theta_n,\Phi_n)$ for a function $\zeta\colon\Re^d \times \zstate\to\Re^d$.

\paragraph{Algorithm}


The Zap SA algorithm \eqref{e:ZapSAalgo} to solve \eqref{e:galerkin}
is obtained on  specifying 
\begin{subequations}
\label{e:safn}
\begin{align}
\clD(\theta_n,\Phi_{n+1}) &\eqdef r(X_n, U_n) + \disc\uQ^{\theta_n}(X_{n+1}) - Q^{\theta_n}(X_n,U_n)  
  \label{e:safn-d} 
  \\
f(\theta_n,\Phi_{n+1}) &\eqdef \clD(\theta_n,\Phi_{n+1}) \elig_n 
\label{e:safn-f} 
\end{align}
\end{subequations}


At points of differentiability, the derivative of $\barf$ has a simple form:
\begin{equation}
\label{e:barfder}
A(\theta) \eqdef \partial_\theta  \barf(\theta)= \Expect[\elig_n \big(\disc\partial_\theta  Q^\theta(X_{n+1}, \phi^{\theta}(X_{n+1})) - \partial_\theta  Q^\theta(X_{n},U_{n}) \big) + \clD(\theta,\Phi_{n+1})  \partial_\theta  \elig_n ]
\end{equation}
The Zap SA algorithm for Q-learning is exactly as described in \eqref{e:ZapSAalgo} with $f$ defined in \eqref{e:safn-f},  and $A_{n+1}(\theta)$ defined to be the term inside the expectation \eqref{e:barfder}:
\begin{subequations}
\begin{align}
A_{n+1} &=  \elig_n[\disc \partial_\theta Q^{\theta_n}(X_{n+1},\phi^{\theta_n}(X_{n+1})) - \partial_\theta  Q^{\theta_n}(X_n,U_n) ] + \clD(\theta_n, \Phi_{n+1}) \partial_\theta  \elig_n 
\label{e:zap-A}\\
\haA_{n+1} &= \haA_n + \stepf_{n+1}  \bigl[  A_{n+1} - \haA_n  \bigr]
\label{e:zap-Ahat} \\
G_{n+1} &= -[\epsy I +\haA_{n+1}^\transpose\haA_{n+1} ]^{-1}\haA_{n+1}^\transpose
\label{e:zap-G} \\
\theta_{n+1}  & = \theta_n +\alpha_{n+1}G_{n+1} f(\theta_n,\Phi_{n+1})
\label{e:zap-theta}
\end{align}%
\label{e:zap}%
\end{subequations}%
Recall that $\phi^{\theta_n}$ is uniquely determined by \eqref{e:phi} in the definition of $A_{n+1}$.

The step-size sequences $\{\alpha_n\}$ and $\{\stepf_n\}$ satisfy standard requirements for two-time-scale SA algorithms \cite{bor08a}: $\stepf_n/\alpha_n\to\infty$ as $n\to \infty$.  For concreteness, in analysis we fix  
\begin{equation}
\alpha_n = 1/(n+n_0),\quad  \stepf_n = \alpha_n^\rho\,,  \qquad n\ge 1\, , \qquad \text{\it with    } n_0 \geq 1 \,, \,\, \rho\in (0.5,1)
\label{e:gains}
\end{equation}
The theory in this paper is focused on decreasing step-size mainly because theory of SA is more mature in this context.   For constant step-size, with $\alpha_n\equiv \alpha$,  $\stepf_n\equiv \stepf$, with $\stepf  = k \alpha$ for fixed $k\gg 1$,  convergence of the algorithm can proceed by viewing the joint process $(\theta_n, \haA_n,\Phi_n)$ as a time-homogeneous Markov chain.     Based on \cite[Theorem 2.3]{bormey00a} and \cite[Chapter 9]{bor08a}, it is conjectured  that there exists    $\bar{\alpha}>0$ such that 
\begin{equation}
\label{e:cs}
\lim_{n\to\infty}
\Expect[\|\theta_n - \theta^*\|^2] = O(\alpha)  \,, \qquad \alpha\in [0, \bar{\alpha}]\, .
\end{equation}
Unfortunately, the mixing time of the Markov chain will increase with decreasing $\alpha$.

\paragraph{Convergence Analysis}

Given that $f$ in \eqref{e:safn-f} is non-smooth in $\theta$, analysis is cast in the theory of generalized subgradients of non-smooth functions.  
Consider first the temporal difference term $\clD:\Re^d\times \zstate\to \Re$. For each $z\in \zstate$, the set of generalized subgradients of $\clD(\theta, z)$ at $\theta_0$ is a convex set of row vectors, denoted by $\partial_\theta \clD(\theta_0, z)$ \cite[Chapter 10]{fra13}.
A vector $\vartheta \in \partial_\theta \clD(\theta_0, z)$ has the defining property,
\begin{equation}
\label{e:subg-s}
  \vartheta v \leq \lim_{s \downarrow 0}\frac{\clD(\theta_0+  s v, z) - \clD(\theta_0, z)}{s} \,, \qquad v\in\Re^d 
\end{equation}
The limit exists because $\clD(\theta, z)$ is the pointwise maximum   of smooth functions~\cite[Theorem 10.22]{fra13}.  
The generalized subgradient of $f(\theta,z)$ exists under additional assumptions.

Recall that $\elig_n = \zeta(\theta_n,\Phi_n)$ for each $n$.  It is assumed henceforth that $\zeta$ is  differentiable in $\theta$.   In the presentation here we impose the additional assumption that  the vector-valued function $\zeta$ has   \textit{non-negative entries}.  We then obtain a version of the chain rule:  
\begin{equation}
\label{e:subg-s-f}
  \partial_\theta f(\theta_0, z) = \{   \zeta(\theta_0,z)  \vartheta  + \clD(\theta_0, z) \partial_\theta  \zeta(\theta_0,z)  :  \vartheta\in  \partial_\theta \clD(\theta_0, z)\} 
\end{equation}
We obtain in \Lemma{t:chain_rule+}
a similar representation for the set of generalized subgradients of $\barf$:
\begin{equation}
\clA(\theta) \eqdef \Bigl\{  A\in \Re^{d\times d}  : 
Av \leq \lim_{s\downarrow 0} \frac{\barf(\theta + sv) - \barf(\theta)}{s} \,, \quad v\in \Re^d  \Bigr \}
\label{e:der_barf1}
\end{equation}
Non-negativity is relaxed in the appendix,  based on a signed decomposition of $\zeta$.
 



\smallskip
\noindent
\textbf{Assumptions:}
\begin{itemize}

\item[{\textbf{(A1)}}] The joint process $(\bfmX, \bfmU)$ is an irreducible Markov chain with unique invariant pmf $\varpi$.

\item[{\textbf{(A2)}}]   $Q$ and $\zeta$ are Lipschitz continuous and twice continuously differentiable in $\theta$; $f(\theta,z)$ is Lipschitz continuous for each $z\in\zstate$;     $\|\barf\|$ is coercive;
$A^\transpose \barf(\theta) \neq 0$ for   $\theta \not\in \Theta^*$, $   A\in \clA(\theta)$.


\item[{\textbf{(A3)}}]  The set $\Theta^*$ is a singleton, so that there is a unique $\theta^*\in \Re^d$ satisfying $\barf(\theta^*)=0$.


\end{itemize}
Assumption A1 rules out $\epsilon$-greedy policies and other parameter-dependent choices.   It is likely that the theory can be extended using the general theory in \cite[Sections~6.2 and 6.3]{karbha16,bor08a}.  

The second and third assumptions are first applied to the ODE  \eqref{e:ZAPODE}:
the coercive assumption in A2 implies boundedness of solutions,  and this with A3   implies global asymptotic stability of \eqref{e:ZAPODE}.


 

It is also assumed throughout the paper that $\{\theta_n\}$ is bounded.  This can be verified under additional assumptions  (which are easily satisfied for linear function approximation), based on an extension of the ODE approximation used to prove \Theorem{t:ZapQ-main}.  Details may be found in Appendix \ref{s:bdd_theta}.  
The following summarizes the main results of this paper: 

\begin{theorem}
\label{t:ZapQ-main}

Let  $\{\theta_n\}$ be the parameter sequence obtained from the
 Zap Q-learning algorithm \eqref{e:zap}, with some fixed $\epsy>0$.   If this sequence is bounded, then
\begin{romannum}

\item If Assumptions~A1--A2 hold, then  $\lim_{n \to \infty}  \barf(\theta_n)=0$   
  a.s.. 
\item If Assumptions~A1--A3 hold, then 
 $\lim_{n \to \infty} \theta_n = \theta^*$ a.s.. 

 \qed
\end{romannum}
\end{theorem}

\paragraph{Convergence Rate}

 Establishing a CLT for the scaled error sequence $\{\sqrt{n} \tiltheta_n\}$ requires a ``tightness bound'' \cite[Chapter 8, Lemma 5]{bor08a} and the following:
\begin{itemize}
\item [\textbf{(A4)}] $f(\theta,z)$ is smooth in a neighborhood of $\theta^*$ for   $z\in\zstate$, and $A(\theta^*) = \partial_\theta \barf(\theta^*)$ is non-singular.
\end{itemize}
Tightness is used to justify an approximation of the algorithm with its linearization \eqref{e:SAlinearized}.
The proof of tightness is left to future work.    In  Appendix \ref{s:var_approx_proof} we consider the linearization, and show that the asymptotic covariance satisfies
  $\Sigma_\theta  = \Sigma_\theta^* +  \epsy^2\Sigma_\theta^{(2)} + O(\epsy^3)$,    where $ \Sigma_\theta  $ is the asymptotic covariance obtained for Zap Q-learning \eqref{e:zap},  $ \Sigma_\theta^* $ is the optimal covariance,   and  $\Sigma_\theta^{(2)} $ is identified in \Prop{t:zapq}.

\paragraph{Overview of Proof of \Theorem{t:ZapQ-main}}  [complete proofs are found in the appendix]

The first step is analysis of the ODE \eqref{e:ZAPODE}  that $\{\theta_n\}$ aims to approximate. It is shown in \Lemma{t:odestable} that  it admits at least one solution from each initial condition.  If in addition (A3) holds, then the ODE is globally asymptotically stable.   These conclusions are obtained through a uniform approximation based on a smooth vector field.    

A significant challenge is establishing solidarity between the ODE and the stochastic recursion.   To illustrate the main ideas, consider the linear parameterization   $Q^\theta = \psi^\transpose \theta$, with $\elig_n$ non-negative so that the chain rule \eqref{e:subg-s-f} holds.
We then obtain subgradients of the components of   $f$:   
\begin{equation}
\label{e:linear-sub}
\begin{aligned}
f(\theta_n+v, \Phi_{n+1}) &= \max_u\big\{  r(X_n, U_n) + [ \disc \psi(X_{n+1}, u) - \psi(X_n,U_n) ]^\trans  (\theta_n + v)\big\}  \elig_n\\
&\geq f(\theta_n, \Phi_{n+1}) + A_{n+1}v \,, \qquad v\in\Re^d
\end{aligned}
\end{equation}
The update equation for $\haA_{n+1}$ in \eqref{e:zap-Ahat} is used to obtain the averaged version of \eqref{e:linear-sub}:
\begin{equation}
\label{e:linear-sub-mean}
\barf(\theta_n+v) \geq \barf(\theta_n) + \haA_{n+1}v + o(1) 
	\,, \qquad v\in\Re^d\,, \  \|v\|\le 1
\end{equation}
where $o(1)\to 0$ as $n\to\infty$, uniformly in $v$. 
 
The arguments are considerably more complex when $Q^\theta$ is non-linear, and the positivity assumption on $\elig_n$ is relaxed.  In particular, without positivity,   neither $f$ nor $\barf$ admit the generalized subgradients.    Fortunately, the techniques developed for the special case can be adapted to demonstrate that $\haA_{n+1}v$   approximates the directional derivative $\barf'(\theta_n; v)$ for each  $v$ and all large $n$.   This leads to a proof that  $\barf(\theta_n)$ shares many attractive properties of $\barf(\Xt_t)$ from ODE \eqref{e:ZAPODE}.   Though we have not been able to establish a true ODE approximation as defined in the traditional sense,  these properties lead to the convergence results in  \Theorem{t:ZapQ-main}.



\section{Numerical Results}
\label{s:numerics}

The   Zap Q-learning algorithm was tested on three examples from OpenAI gym: Mountain Car, Acrobot, and Cartpole \cite{openAIgym}. 
The   approximation of $Q^\theta$ was obtained based on a neural network, so that  the parameter $\theta\in \Re^d$ represents weights in the neural network. 
Rather than achieving the best score for specific tasks, the objective of the experiments surveyed in this section was to investigate the stability and consistency of the Zap Q-learning algorithm across different domains, and  varying neural network sizes.   Common in each experiment:  a feedforward neural network that is 
fully connected, 
using the Leaky ReLU activation function.

\begin{figure}[ht]
	\centering
	\includegraphics[width=\textwidth]{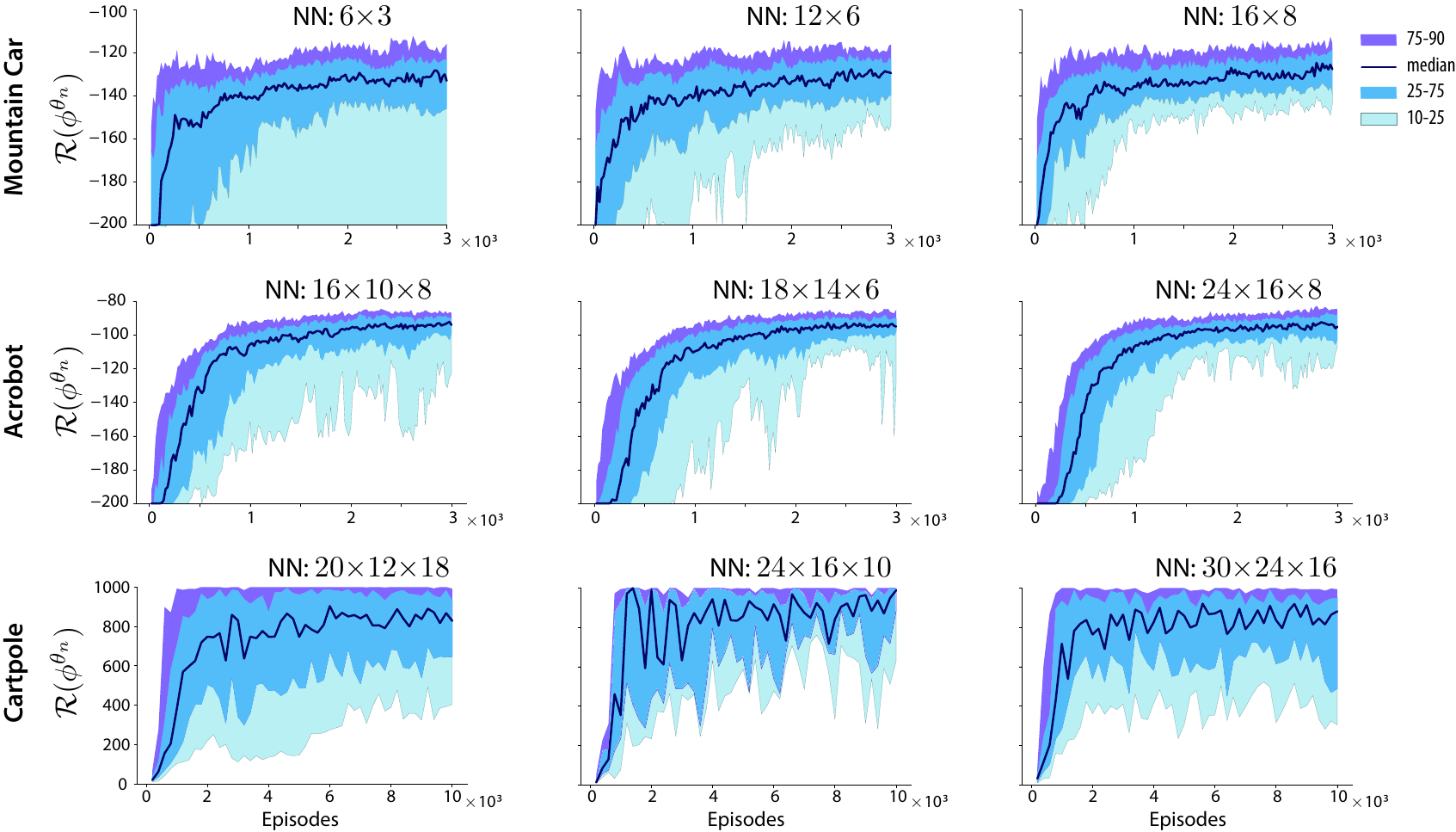}

	\caption{Average rewards for the three examples, shown by percentile. }
	\label{f:examples}
\end{figure}


The goal in each of the three examples is to collect as many rewards as possible before the state reaches a terminal set denoted $\termState \subset \state$.  
To avoid infinite values we introduce a deterministic upper bound $\bar{\tau}\ge 1$,
and consider the bounded horizon $\tau = \min(\bar{\tau}, \tau_\termState)$ with
$\tau_\termState = \min\{n\ge 1 :X_n \in \termState\}$.
The   Q-function is denoted
\[
Q^*(x,u) \eqdef \max \Expect[\sum_{n=0}^{\tau -1} r (X_n, U_n)\mid  X_0=x \,, U_0 = u]\, , 
\]
which for $\bar{\tau}=\infty$ solves the Bellman equation
\begin{equation}
\label{e:bellman-sp}
Q^*(x,u) =  r(x,u) + \Expect[\ind\{X_1 \notin \termState\} \uQ^*(X_1)\mid X_0=x, U_0 = u]
\end{equation}
Following the    roadmap outlined in \Section{sec:zap-q-learning},  we seek an approximate solution to \eqref{e:bellman-sp}
 based on a root-finding problem analogous to \eqref{e:galerkin}:  find $\theta^*$ such that
\begin{equation}
\label{e:galerkin-ssp}
\barf(\theta^*) = 0\,, \quad \text{with}\quad \barf(\theta) \eqdef \Expect \bigl[\elig_n\bigl(r(X_n, U_n) +\ind\{X_{n+1} \notin \termState\} \uQ^{\theta}(X_{n+1}) - Q^{\theta}(X_n,U_n)\bigr) \bigr]
\end{equation}
where the distribution of $X_0$ is given. 
The Zap-Q algorithm \eqref{e:zap} is easily adapted to the modified definition of $\barf$ in \eqref{e:galerkin-ssp}. 

The performance of  the greedy policy $\phi^\theta$ induced by $Q^{\theta}$ is denoted
\begin{equation}
\label{e:avr-reward}
\clR(\phi^\theta) \eqdef \Expect \Big[\sum_{n=0}^{\tau-1} r(X_n, \phi^\theta(X_n)) \Big]
\end{equation}
Specifics of the meta-parameters and the details of how \eqref{e:avr-reward} was estimated are contained in \Section{s:numDetails}.  

\begin{figure}[h!]
	
	\centering
	\includegraphics[width=\textwidth]{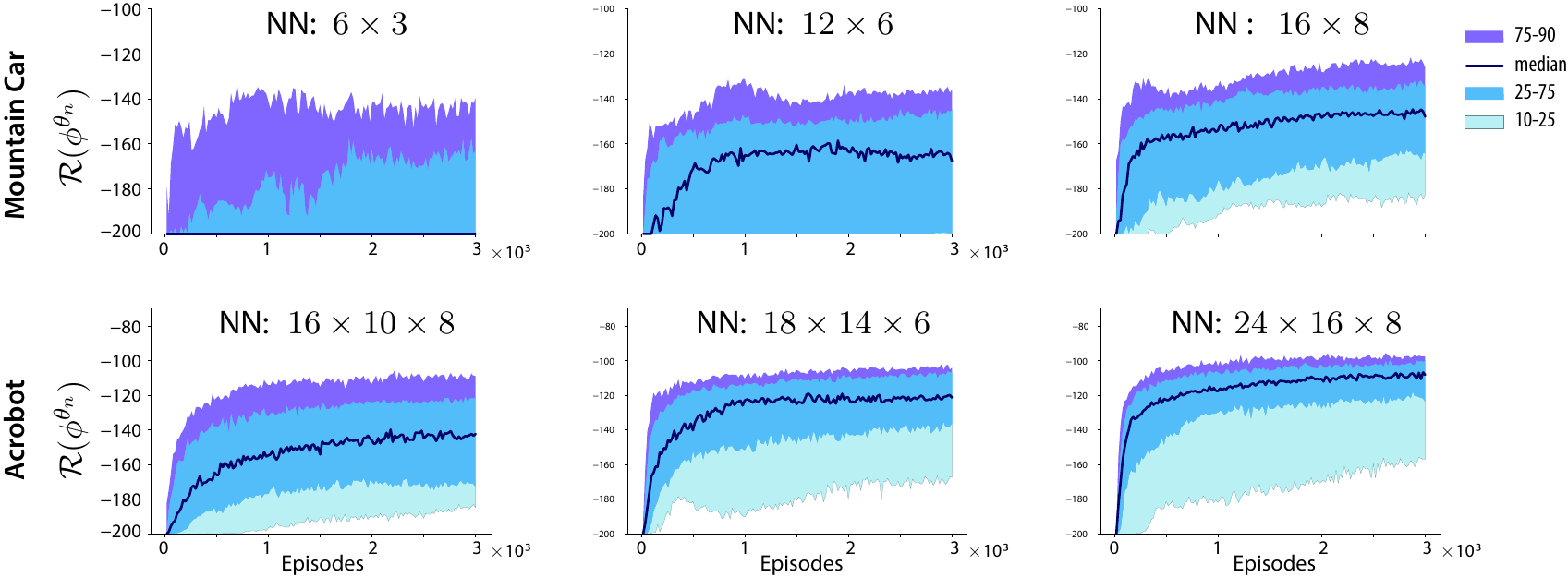}
	\caption{Average rewards with decreasing step-size, shown by percentile. }
	\label{f:ds}
\end{figure}

Two minor modifications of the algorithm were used in these experiments to reduce complexity:  
\\
 \textit{1.  Periodic gain update.}
 An integer $N_d > 1$ was fixed,   and the gain $G_{n}$ appearing in  \eqref{e:zap-theta} was updated only for   integer multiples of $N_d $.
 In particular, the matrix inversion step was only performed at these iterations.   
Letting $N$ denote the total number of iterations, the overall complexity of running this algorithm for $N$ iterations is in the worst-case $O(Nd^{3}/N_d + Nd^2)$  (see  Appendix \ref{s:numDetails} for further discussion on complexity).  
We observed that $N_d=50$ worked well for all experiments, and the performance was unchanged from $N_d = 1$.   

 \textit{2.  Periodic eligibility update.}
The definition of the eligibility vector in \eqref{e:elig} was modified to break $\theta$-dependency: $\elig_n \eqdef \nabla_\theta Q^{\theta_{i(n)}}(X_n, U_n)$,  with   $i(n) \eqdef (\lfloor n/N_\zeta + 1\rfloor - 1)N_{\zeta} $,  and $N_\zeta = 2000$. 
  We then ignored the term $\clD(\theta_n,\Phi_{n+1})\partial_\theta \elig_n$   when computing $A_{n+1}$ based on    \eqref{e:zap-A}.   For comparison, we performed experiments in which this term was included  (increasing complexity considerably since  second derivatives are required \cite{bunwei94}), and saw no improvement in performance.


Experiments were performed with both decreasing step-size (defined in \eqref{e:gains}), and constant step-size. We found that constant step-size implementations were more reliable for the Mountain car and Acrobot examples, and the diminishing step-size gave better results for Cartpole.  	\Figure{f:examples} shows  results obtained for these choices, and selected results with diminishing step-sizes are contained in \Figure{f:ds}. The size of neural networks indicated in the figure
  refers only to hidden layers.  To obtain the quantiles shown,  each experiment was repeated 50 times, with parameters randomly initialized by the Kaiming uniform method \cite{hezharensun15, pasgro17}. The first column   shows results from the smallest network for which we obtained reliable results for the particular example.







\section{Conclusions}

Zap Q-learning   is provably consistent with nonlinear function approximation, under very general conditions.   Theoretical questions remain, such as extension to more general exploration strategies, and convergence properties for more general step-size rules.  There are also architectural questions.  For example, the definition of the eligibility vector \eqref{e:elig} is not sacred.   Better overall performance, and simpler conditions for convergence may be achieved through alternatives (there are obvious choices for deterministic control systems rather than MDPs). 

Better algorithm design requires a better theory for function approximation in Q-learning.   The parameter estimates in both Zap Q-learning and the basic algorithm  \eqref{e:Q0} are solutions to the root-finding problem \eqref{e:galerkin}.   How do we bound the Bellman error,  or the absolute error $|Q^{\theta^*} - Q^*|$?     The goal is to create a theory as complete as TD-learning with linear function approximation,  for which vector space concepts bring crisp answers \cite[Theorem 1]{tsiroy97a}.





The matrix-inversion step in the algorithm  may be a barrier to application of Zap-Q in some problems.  A simple approach to reduce complexity is described in the numerical results, and we expect  to obtain much more efficient implementations, perhaps by applying a distributed implementation \cite{anigupkorregsin20},  and adapting techniques from stochastic optimization.

We are currently exploring the application of the techniques in this paper to analyze Deep Q-learning \cite{volkordav15},  and application of Zap-SA to actor-critic methods.

\clearpage 

\bibliographystyle{abbrv}
\bibliography{strings,markov,q,NIPS19extras}  

\clearpage 
\appendix


\section{Appendices}

\subsection{Establishing boundedness of the parameter estimates}
\label{s:bdd_theta}

Suppose that the following limits exist:
\[
\begin{aligned}
Q^\theta_\infty (x,u)  &=  \lim_{m\to\infty}  m^{-1}  Q^{m\theta}(x,u) \,,\qquad x\in\state\,,\   u\in\action
   \\
\elig_\infty (\theta, z)  &=  \lim_{m\to\infty}  \elig  (m\theta, z)  \,,\qquad z\in \zstate
\end{aligned} 
\]
where the limiting functions are twice continuously differentiable.    The global Lipschitz conditions in (A2) imply that the gradients also converge, and the convergence is uniform on compact sets.    We then obtain a vector field for the ``ODE at infinity'' introduced in \cite{bormey00a}:
\[
\barf_\infty(\theta) \eqdef  \lim_{m\to\infty}  m^{-1} 
\barf(m\theta)= \Expect \big [  \big( \disc\uQ^{\theta}_\infty(X_{n+1}) - Q_\infty^{\theta}(X_n,U_n) \big) \elig_\infty(\theta, \Phi_n)\big]
\]
and a similar definition for $ f_\infty(\theta, z) $.  The associated regularized  Newton-Raphson flow ``at infinity'' is similar to \eqref{e:ZAPODE}:
\begin{equation} 
\ddt \Xt_t =    - [\epsy I +A_\infty(\Xt_t) ^\intercal A_\infty(\Xt_t) ]^{-1}  A_\infty(\Xt_t) ^\intercal     \barf_\infty(\Xt_t)  \,,   \qquad  A_\infty(\Xt_t)=  \partial_\theta \barf_\infty\,  ( \Xt_t)
\label{e:ZAPODE_inf}
\end{equation}

With $\clA_\infty(\theta)$ defined as in \eqref{e:der_barf1} with respect to $\barf_\infty$,  assume the following:

\smallskip

\noindent
 \textbf{(A2${}_\infty$)} 
  The functions $Q_\infty$ and $\elig_\infty$ are Lipschitz continuous and twice continuously differentiable in $\theta$ in any open set not containing the origin;
  $f_\infty(\theta,z)$ is Lipschitz continuous for each $z\in\zstate$;     
$A^\transpose \barf_\infty(\theta) \neq 0$ for all $\theta \not=0$ and $   A\in \clA_\infty(\theta)$.

Assumptions (A2) and (A2${}_\infty$) are identical when the function approximation $Q^\theta$ is linear, and $\zeta =\nabla Q^\theta$.

The function $\|\barf_\infty\|$ is coercive under (A2${}_\infty$) since $\barf_\infty$ is radially linear:    $\barf_\infty(m \theta) = m \barf_\infty(  \theta) $ for any $\theta$ and any $m\ge 0$.   \Prop{t:odestable} can be adapted to show that \eqref{e:ZAPODE_inf} is globally asymptotically stable.  
\cite[Sections~6.3, Theorem 9]{bor08a} explains how stability of the ODE implies stability of the SA algorithm.

\subsection{Asymptotic covariance of regularized Zap SA}
\label{s:var_approx_proof}

We first introduce a standard result in linear system theory \cite[Theorem 2.6-1]{kailath1980linear}.

\begin{lemma}
\label{t:lyapunov}
If $A\in \Re^{d\times d}$ is Hurwitz and $\Sigma_{\Delta} \in \Re^{d\times d}$ is positive semi-definite, then there exists a unique solution $\Sigma \geq 0$ that solves the Lyapunov equation,
\[ 
A\Sigma + \Sigma A^{\transpose} + \Sigma_{\Delta}  =0 \,,
\]
whose solution can be expressed
\[
\Sigma = \int_0^{\infty} \exp(A\tau)\Sigma_{\Delta}\exp(A^{\transpose} \tau) \, d \tau
\]
\end{lemma}

Let $\{\clE_n^G\}$ be the sequence obtained by the stochastic linear recursion \eqref{e:SAlinearized} with matrix gain $G \in \Re^{d\times d}$:
\begin{equation}
\label{e:SAlinearized-G}
\clE_{n+1}^G  = \clE_n^G + \alpha_{n+1} G[A_*  \clE_n^G  + \Delta_{n+1}  ]  \,, \qquad {\clE_0^G = \theta_0 -\theta^*}
\end{equation}
Denote the asymptotic covariance of $\{\clE_n^G\}$ by $\Sigma_\theta^G \eqdef \lim_{n\to\infty} n\Expect[\clE_{n}^G(\clE_{n}^G)^\transpose]$. According to the eigenvalue test \eqref{e:eigTest}, $\Sigma_\theta^G$ is finite if  $\half I + GA_*$ is Hurwitz. 
It is well known that the matrix gain $G_* = -A_*^{-1}$ achieves the minimal asymptotic covariance:    
\begin{equation}
\label{e:opt-asym}
\Sigma_\theta^* = A_*^{-1}\Sigma_{\Delta}(A_*^{-1})^\transpose
\end{equation}
 The following result is standard in stochastic approximation \cite[Part I, Section 3.2.3, Proposition 4]{benmetpri90},  and quantifies optimality of $\Sigma_\theta^* $.  
\begin{lemma}
\label{t:asy-cov-gap}
\begin{romannum}
Suppose that $A_*$ and  $\half I + GA_*$ are Hurwitz.   Then,
\item The asymptotic covariance $\Sigma_\theta^G \geq 0$ uniquely solves the Lyapunov equation:
\[
(\half I + GA_*)\Sigma_\theta^G + \Sigma_\theta^G(\half I + GA_*)^\transpose + G\Sigma_{\Delta}G^\transpose = 0
\]
 
\item  The sub-optimality gap $\tilSigma_\theta^G = \Sigma_\theta^G - \Sigma_\theta^* \ge 0$ uniquely solves the Lyapunov equation:
\begin{equation}
\label{e:lyapunov-gap}
(\half I + GA_*)\tilSigma_\theta^G + \tilSigma_\theta^G(\half I + GA_*)^\transpose + (G+A_*^{-1})\Sigma_{\Delta}(G+A_*^{-1})^\transpose = 0
\end{equation}
\end{romannum}
\end{lemma}

For  any symmetric matrix $S\in \Re^{d\times d}$, denote by $\lambda(S)$ the set of its eigenvalues.
\begin{proposition}
\label{t:zapq}
Suppose $A_*\in \Re^{d\times d}$ is Hurwitz, and denote $G_\epsy = -[\epsy I + A_*^\intercal A_*]^{-1}A_*^\intercal$. If $\,0<\epsy<\lambda_{\min}(A_*^\transpose A_*)$, then
 $\half I + G_\epsy A_* $ is Hurwitz,   so that the matrix gain $G_\epsy$ in the linear recursion \eqref{e:SAlinearized-G} results in a finite asymptotic covariance $\Sigma_\theta^\epsy$.   Moreover, the follow approximation holds:
\begin{equation}
\Sigma_\theta^\epsy=\Sigma_\theta^* +  \epsy^2\Sigma_\theta^{(2)} + O(\epsy^3)  \,,\qquad \textit{with   $\Sigma_\theta^{(2)}=(A_* A_*^\intercal A_* )^{-1} \Sigma_\Delta (A_*^\intercal A_*A_*^\intercal  )^{-1}$.
}
\label{e:TS_Sigma}
\end{equation} 
\end{proposition}

\begin{proof}
The set of eigenvalues of $\half I + G_\epsy A_*$ admits the following representations:
\[
\begin{aligned}
	\lambda(\half I + G_\epsy A_*) &=  \big\{ \frac{1}{2}  -\lambda:\lambda\in \lambda([\epsy I + A_*^\intercal A_*]^{-1}A_*^\intercal A)\big\} \\
	&= \big\{\frac{1}{2} - \frac{1}{\lambda}:\lambda\in\lambda(\epsy (A_*^\intercal A_*)^{-1} + I)\big\}\\
	&= \big\{\frac{1}{2} - \frac{1}{\epsy \lambda + 1}:\lambda\in\lambda((A_*^\intercal A_*)^{-1}) \big\}\\
	&= \big\{\frac{1}{2} - \frac{1}{\epsy/ \lambda + 1}:\lambda\in\lambda(A_*^\intercal A_*)\big\}
\end{aligned}
\]
Given $0<\epsy<\lambda_{\min}(A_*^\transpose A_*)$, the eigenvalues of $\half I + G_\epsy A_* $ are real and strictly negative.  In particular, this matrix is Hurwitz, as claimed.

We next establish the approximation \eqref{e:TS_Sigma}. By \Lemma{t:asy-cov-gap} (iii), $\tilSigma_\theta^\epsy = \Sigma_\theta^\epsy - \Sigma_\theta^*$ solves the  Lyapunov equation \eqref{e:lyapunov-gap} with $G$ replaced by $G_\epsy$. Denoting $\tilG_\epsy = G_\epsy+ A_*^{-1} $, we obtain by \Lemma{t:lyapunov},
\begin{equation}
\label{e:covIncreReal}
\begin{aligned}
	\tilSigma_\theta^\epsy = &\int_0^\infty \textstyle \exp([\half I + G_\epsy A_*] \tau ) \tilG_\epsy\Sigma_\Delta \tilG_\epsy^\transpose   \textstyle\exp([\half I + G_\epsy A_*]^\transpose\tau ) \, d\tau \\
\end{aligned}
\end{equation}
A Taylor series representation of matrix inverse  results in the following:
\[
\begin{aligned}
	G_\epsy A_* &= -[\epsy I + A_*^{\transpose}A_*]^{-1}A_*^{\transpose}A_*= -[ I + \epsy (A_*^{\transpose}A_*)^{-1}]^{-1} = -[I - \epsy A_* ^{-1} (A_*^{\transpose})^{-1}]+ O(\epsy^2)\\
	\tilG_\epsy &= G_\epsy + A_*^{-1}  = [G_\epsy A_* + I]A_*^{-1} =\epsy A_* ^{-1} (A_*^{\transpose})^{-1} A_* ^{-1}+ O(\epsy^2)
\end{aligned}
\]
	With $\Sigma_\theta^{(2)}=(A_* A_*^\intercal A_* )^{-1} \Sigma_\Delta (A_*^\intercal A_*A_*^\intercal  )^{-1}$, the integral in \eqref{e:covIncreReal} becomes
\begin{equation}
\label{e:covInte2}
\begin{aligned}
	\tilSigma_\theta^\epsy \!=\!\epsy^2\int_0^\infty\textstyle \exp (-[\half I\! -\! \epsy A_* ^{-1}(A_*^{\transpose})^{-1}]\tau )\Sigma_\theta^{(2)} \textstyle\exp (-[\half I\!-\!\epsy A_* ^{-1}(A_*^{\transpose})^{-1}]^\transpose\tau )d\tau + O(\epsy^3)
\end{aligned}
\end{equation}
	Another Taylor series expansion for matrix exponential gives
\[
	\textstyle\exp(-[\half I -\epsy A_* ^{-1}(A_*^{\transpose})^{-1}]\tau ) = \textstyle\exp (-\frac{\tau}{2} I) +\epsy \tau A_* ^{-1}(A_*^{\transpose})^{-1}  \exp (-\frac{\tau}{2} I) + O(\epsy^2)
\]
	Consequently, the integral in \eqref{e:covInte2} can be rewritten as
\[
\begin{aligned}
	\tilSigma_\theta^\epsy =& \epsy^2 \int_0^\infty \textstyle\exp(-\frac{\tau}{2} I )\Sigma_\theta^{(2)} \textstyle \exp(-\frac{\tau}{2} I) d\tau+ O(\epsy^3) \\
	=& \epsy^2 \Sigma_\theta^{(2)} + O(\epsy^3)
\end{aligned}
\]
\end{proof}

\subsection{Eigenvalue test for GQ-learning}
\label{sec:analysis-gq}

Consider the linear function approximation architecture: $Q^\theta(x,u)=\psi(x,u)^\transpose \theta$, where $\psi:\state\times\action\to \Re^d$ is the basis function. With eligibility vector $\elig_n \eqdef \psi(X_n, U_n)$, 
let $\barf$ be defined by \eqref{e:galerkin}. GQ-learning \cite{maeszebhasut10}  aims to solve 
the root finding problem \eqref{e:galerkin}, transformed into the non-convex optimization problem  \eqref{e:gq-obj}.

The GQ-learning \cite{maeszebhasut10}  algorithm is the two-time scale SA algorithm, 
\begin{subequations}
\begin{align}
\theta_{n+1} &= \theta_n + \alpha_{n+1}[\clD(\theta_n, \Phi_{n+1})\elig_n- \disc \varphi_{n+1}^\transpose \elig_n \psi(X_{n+1}, \phi^{\theta_{n}}(X_{n+1}))] 
\label{e:slow-gq}
\\
\varphi_{n+1} &= \varphi_n + \stepf_{n+1} \elig_n[\clD(\theta_n, \Phi_{n+1})- \psi(X_n, U_n)^\transpose \varphi_n] \label{e:fast-gq} 
\end{align}%
\label{e:gq-alog}%
\end{subequations}%
where $\{\stepf_n\}$ and $\{\alpha_n\}$ are non-negative step-size sequences satisfying $\alpha_n/\stepf_n\to 0$ as $n\to \infty$. 
The fast time scale recursion \eqref{e:fast-gq} for $\{\varphi_n\}$  is designed so that $\varphi_n \approx M\barf(\theta_n)$ for large $n$.   
It follows that the ODE approximation of \eqref{e:slow-gq} is  \eqref{e:GQODE}.

%
%

\begin{proposition}
\label{t:gq-infinite} 
The linearization matrix   for GQ-learning at $\theta^*$ is given by $A_{\text{GQ}} = -A_*^\intercal MA_*$, whenever $\theta^*$ is a solution to $A(\theta)^\intercal M\barf(\theta) =0$.  With the tabular basis: $\psi_k(x, u) = \ind\{ x= x^k, u = u^k\} \,,  1\leq k \leq \ell_x\cdot\ell_u$, there is an eigenvalue $\lambda_{GQ}$ of $A_{\text{GQ}}$ satisfying
\[
\lambda_{GQ} \geq -(1-\disc)^2
\]
\end{proposition}

We first introduce some notation for tabular Q-learning.  For any deterministic stationary policy $\phi:\state \to \action$, let $S_\phi$ denote the substitution operator, defined for any function $Q:\state\times\action\rightarrow \Re$ by $S_\phi Q(x) =Q(x,\phi(x))$. With $P$ viewed as a matrix with $\ell_x\cdot\ell_u$ rows and $\ell_x$ columns, $PS_{\phi}$ can be interpreted as the transition matrix for the joint process $(\bfmX, \bfmU)$ when $\bfmU$ is defined using policy $\phi$~\cite{devmey17a}. Then
 $\barf(\theta)$ can be written in matrix form
 \begin{equation}
 \label{e:tabular-f}
 \barf(\theta) = \Pi r + \Pi[\disc PS_{\phi^\theta} - I]\theta
 \end{equation}
 where $\Pi$ is a diagonal matrix with entries: $\Pi(k,k) \eqdef\varpi(x^k,u^k)$ and $r$ is a vector with entries: $r(k)\eqdef r(x^k, u^k)$.
 The derivative $A(\theta)$ of $\barf(\theta)$ is given by
 \[
 A(\theta) = \Pi[\disc PS_{\phi^\theta} - I]
 \]

\begin{proof}
The matrix $A_{\text{GQ}}$ is the derivative of $-A(\Xt_t)^\intercal M\barf(\Xt_t)$ at $\theta^*$. For the tabular case, by~\eqref{e:tabular-f},
\[
	A_{\text{GQ}}=-[\disc PS_{\phi^*}-I]^\transpose\Pi[\disc PS_{\phi^*}-I]  = - H^\transpose H\,, 
\]
with $H \eqdef \Pi^{1/2}[I -  \disc PS_{\phi^*}]$. It suffices to show that   $H^\transpose H$ has a positive eigenvalue less than $(1-\disc)^2$. 
	
Since $H^{-1}$ is a positive and irreducible matrix, we can apply the same arguments as in \cite[Theorem 3.3]{devmey17a} to bound the
Perron-Frobenius eigenvalue as follows:  
\[
	\lambda_{\text{PF}} \geq \frac{1}{1-\disc} \min_{x,u} \frac{1}{\sqrt{\varpi(x,u)}}
\]
Therefore, $H$ has positive eigenvalue $\lambda_H = \lambda_{\text{PF}}^{-1}$ such that
\[
	\lambda_H \le (1-\disc)\max_{x,u}\sqrt{\varpi(x,u)}
\]
Applying \cite[Theorem 5.6.9]{horjoh12} we obtain the complementary bound
\[
	\lambda_H\geq \sigma_{\min}(H) =\sqrt{ \lambda_{\min}(H^\transpose H) }
\]
and combining the two implies:
\[
	\lambda_{\min}(H^\transpose H) \leq  \lambda_H^2\le  (1-\disc)^2\bigl(\max_{x,u}\sqrt{\varpi(x,u)} \bigr)^2  \le (1-\disc)^2 
\]
\end{proof}

\subsection{ODE Analysis}
\label{s:proofode}

To obtain the existence of a solution to  \eqref{e:ZAPODE}, we first consider an ideal smooth setting:
\begin{proposition}
\label{t:ZapSA}
Consider the following conditions for the function $\barf$:  
\begin{alphanum}
\item $\barf$ is globally Lipschitz continuous and continuously differentiable. Hence $A(\cdot)$ is a bounded matrix-valued function.
\item   $\| \barf \|$ is coercive.  That is,   $\{ \theta : \| \barf(\theta)\| \le n\}$ is compact for each $n$.
\item  The function $\barf$ has a unique zero $\theta^*$,   and $A^\intercal(\theta)\barf(\theta) \neq 0$ for $\theta\neq \theta^*$.   Moreover,  the matrix $A_*=A(\theta^*)$ is non-singular.
\end{alphanum}
The following   hold for solutions to the ODE \eqref{e:ZAPODE} under increasingly stronger assumptions:
\begin{romannum}
\item
If (a) holds then  for each $t$,  and each initial condition 
\begin{equation}
\ddt \barf(\Xt_t) = - A(\Xt_t) [\epsy I + A(\Xt_t)^\intercal A(\Xt_t)]^{-1}A(\Xt_t)^\intercal\barf(\Xt_t) 
\label{e:ddtf}
\end{equation}

\item If in addition (b) holds, then the solutions to the ODE are bounded,   and   
\begin{equation}
\lim_{t\to\infty}   A(\Xt_t)^\intercal\barf(\Xt_t) =0  
\label{e:ddtf_b}
\end{equation}

\item If (a)--(c) hold, then    \eqref{e:ZAPODE} is globally asymptotically stable. \qed
\end{romannum}
\end{proposition}

\begin{proof}
The result (i) follows from the chain rule and the definitions.  

The proof of (ii)  is based on the Lyapunov function $V(\Xt) = \frac{1}{2}\|\barf(\Xt)\|^2$ combined with (a):
\[ 
  \ddt V(\Xt_t) = -\barf(\Xt_t)^\intercal A(\Xt_t)[\epsy I +
  A(\Xt_t)^\intercal A(\Xt_t)]^{-1}A(\Xt_t)^\intercal\barf(\Xt_t) 
\]
The right hand side is non-positive when $\Xt_t\neq \theta^*$.   Integrating each side gives for any $T>0$,
\begin{equation}
V(\Xt_T) =  V(\Xt_0)    - \int_0^T\barf(\Xt_t)^\intercal A(\Xt_t)[\epsy I + A(\Xt_t)^\intercal A(\Xt_t)]^{-1}A(\Xt_t)^\intercal\barf(\Xt_t) \, dt
\label{e:Vint}
\end{equation}
so that  $  V(\Xt_T) \le  V(\Xt_0) $ for all $T$.   
Under the coercive assumption, it follows that solutions to  \eqref{e:ZAPODE}  are bounded.   
Also,  letting $T\to\infty$, we obtain from \eqref{e:Vint} the bound
\[
\int_0^\infty\barf(\Xt_t)^\intercal A(\Xt_t)[\epsy I + A(\Xt_t)^\intercal A(\Xt_t)]^{-1} A(\Xt_t)^\intercal\barf(\Xt_t) \, dt \le  V(\Xt_0)  
\]
This combined with boundedness of $\Xt_t$ implies that $\lim_{t\to\infty}A(\Xt_t)^\intercal\barf(\Xt_t)  =0$.

We next prove (iii).    
Global asymptotic stability  of  \eqref{e:ZAPODE}  requires that solutions converge to $\theta^*$ from each initial condition, and also that $\theta^*$ is stable in the sense of Lyapunov \cite{kha02}.   
Assumption (c) combined with (ii)  gives the former, that
$\lim_{t\to\infty} \Xt_t = \theta^*$.   A convenient sufficient condition for the latter is obtained by considering  
$
A_\infty =  \partial_\theta  [\clG(\theta) \barf(\theta)] \mid_{\theta=\theta^*}  
$.
Stability in the sense of Lypaunov holds if this matrix is Hurwitz (all eigenvalues are in the strict left half plane in  $\Co$)  \cite[Thm.~4.7]{kha02}.   Apply the definitions, we obtain $A_\infty = -  [\epsy I + M]^{-1}  M$ with $M = A(\theta^*)^\intercal A(\theta^*) > 0$ (recall that $A(\theta^*)$ is assumed to be non-singular).    The matrix  $A_\infty $ is negative definite, and hence Hurwitz.

\end{proof}

\Prop{t:ZapSA} cannot be applied to the ODE \eqref{e:ZAPODE} that motivated Zap Q-learning
because $\barf$ is only piecewise smooth.    To obtain an extension we consider the ODE  in its integral form: 
\begin{equation}
\label{e:inteode}
\Xt_t = \Xt_0 - \int_0^t [\epsy I + A^\transpose(\Xt_\tau )A(\Xt_\tau)]^{-1}A^\transpose(\Xt_\tau)\barf(\Xt_\tau)\, d\tau, \qquad  t\ge 0
\end{equation}
where   $\barf(\theta)$, $A(\theta)$ are defined in (\ref{e:galerkin}, \ref{e:barfder}).


\begin{proposition}
\label{t:odestable}
Under Assumptions A1-A2, there exists a solution to  \eqref{e:inteode} from each initial condition.   The following hold for any solution:   
\begin{romannum}
\item
		$\displaystyle
		\barf(\Xt_t) =\barf(\Xt_0) - \int_0^t A(\Xt_\tau )[\epsy I + A(\Xt_\tau)^\transpose A(\Xt_\tau)]^{-1}A(\Xt_\tau)^\transpose \barf(\Xt_\tau)\, d\tau , \qquad  t\ge 0
		$
		
\item
		$\|\barf(\Xt_t) \|$ is non-increasing,  and $\lim_{t\to\infty} \barf(\Xt_t) =0$.
		
\item 	If in addition A3 holds, then the ODE \eqref{e:inteode} is globally asymptotically stable. 
		
\end{romannum}
	
\end{proposition}


The proof of existence is obtained by considering smooth approximations of \eqref{e:ZAPODE}.
\begin{lemma}
\label{t:existode}
Under Assumptions A1-A2, there exists a solution to  \eqref{e:inteode} from each initial condition.  
\end{lemma}

\begin{proof}
Define a $C^\infty $  probability density $\eta$ on $ \Re^d$ via  
\begin{equation}
\label{e:mollifiers}
\eta(x) :=
\begin{cases}
k\exp( - (1-\|x\|^2)^{-1}) & \|x\| < 1, \\
0 & \|x\| \geq 1,
\end{cases}
\end{equation}
where $k>0$ is a normalization constant: $\int \eta(x) \, dx = 1$.
For each $\delta>0$, a $C^\infty$ vector field is defined via the convolution:
\begin{equation}
\barf_\delta(x) = \fraction{1}{\delta^{d}}   \int \barf(x-y)  \eta(y/\delta)\, dy  \,,\qquad x\in\Re^d
\label{e:barf-smooth}
\end{equation}
The family of functions $\{\barf_\delta : 0<\delta\le 1\}$  is globally uniformly Lipschitz continuous, with the same Lipschitz constant $b_L$ as of $\barf$. It is also evident that  $\lim_{\delta\downarrow 0}\barf_\delta = \barf$ pointwise. The uniform Lipschitz continuity implies that the convergence is uniform on compact sets.

Denote   $A_\delta(\theta)=\partial_\theta \barf_\delta(\theta)$, and consider the ODE \eqref{e:inteode} with $\barf$ and $A$ replaced by their smooth approximations:
\begin{equation} 
\Xt_t^\delta = \Xt_0^\delta - \int_0^t [\epsy I + A_\delta^\transpose(\Xt_{\tau}^\delta)A_\delta(\Xt_t^\delta)]^{-1}A_\delta^\transpose(\Xt_t^\delta)\barf_\delta(\Xt_t^\delta)\, d\tau ,\qquad \Xt_0^\delta= \Xt_0 \label{e:qf-epsy}
\end{equation}
The solution exists and is unique for each $\delta \in(0,1]$.    To obtain bounds on the solution we require bounds on the matrices involved, and opt for the spectral norm: 
\[
\begin{aligned}
	\|[\epsy I + A_\delta^\transpose(\Xt_t^\delta)A_\delta(\Xt_t^\delta)]^{-1}\| &= \frac{1}{\lambda_{\min}(\epsy I + A_\delta^\transpose(\Xt_t^\delta)A_\delta(\Xt_t^\delta))} \leq \frac{1}{\epsy}
	\\
	\| A_\delta(\Xt_\tau^\delta) \| &\le b_L
\end{aligned}
\]
where $b_L$ is the  Lipschitz constant for $\barf_\delta$. Therefore,
\[
\begin{aligned}
	\|\Xt_t^\delta\| &\leq \|\Xt_0^\delta\| + \int_0^t  \|[\epsy I + A_\delta^\transpose(\Xt_\tau^\delta)A_\delta(\Xt_\tau^\delta)]^{-1}\|\cdot\|A_\delta(\Xt_\tau^\delta)\|\cdot\|\barf_\delta(\Xt_\tau^\delta)\|\, d\tau
	\\
	&\leq \|\Xt_0^\delta\| +\frac{b_{L}}{\epsy} \int_0^t  \|\barf_\delta(\Xt_\tau^\delta) \| \, d\tau
	\\
	&\leq \|\Xt_0^\delta\| + \frac{b_{L}}{\epsy}\int_0^t  \|\barf_\delta(\Xt_0^\delta)\| + \|\barf_\delta(\Xt_\tau^\delta) -\barf_\delta(\Xt_0^\delta) \|  \, d\tau
	\\
	& \leq \|\Xt_0^\delta\| + \frac{b_{L}}{\epsy} \bigl\{ T \|\barf_\delta(\Xt_0^\delta)\|  + b_L\int_0^t   \|\Xt_\tau^\delta -\Xt_0^\delta\|\, d\tau \bigr\}\\
	& \leq  \|\Xt_0^\delta\| + \frac{b_{L}}{\epsy} \bigl\{ T (\|\barf_\delta(\Xt_0^\delta)\| + b_L \|\Xt_0^\delta\|) + b_L\int_0^t   \|\Xt_\tau^\delta \|\, d\tau \bigr\}
\end{aligned}
\]
The set $\{\|\barf_\delta(\Xt_0^\delta)\|: 0 < \delta \leq 1\}$ is bounded by $ \max_{y\in \clB(\Xt_0, 1)}\|\barf(y)\|$, where $ \clB(\Xt_0, 1)$ denotes the closed unit ball in $\Re^d$ centered at $\Xt_0$. By Gronwall's inequality, there exist constants $C_1$ and $C_2$ such that
\[
	\|\Xt_t^\delta\|\leq C_1 + C_2e^{ b_L^2T/\epsy}, \quad t\in[0,T]\,, \qquad \delta \in (0,1]
\]
	This combined with \eqref{e:qf-epsy} implies that $\{\Xt^\delta: 0 < \delta \leq 1\}$ is uniformly bounded and equicontinuous. By the Arzel\`a-Ascoli theorem, there exists a sequence $\delta_n \downarrow 0$ and a continuous function $\Xt^0: [0,T]\to\Re^d$ such that
\[
		\lim_{n \to \infty}\sup_{t\in[0,T]}\|\Xt_t^{\delta_n}-\Xt_t^0\| = 0
\]
	So the functional equation \eqref{e:qf-epsy} holds for $\Xt^0$  with $\delta = 0$, and $\Xt^0$ is thus a solution of \eqref{e:inteode}.
\end{proof}
The following result has been derived in \cite[Lemma A.10]{devmey17a}. We present it here for completeness.  
\begin{lemma}
\label{t:chain1}
	Let $G(\theta)\eqdef \max_{1\leq i\leq \ell_u}G_i(\theta)$ where each $G_i:\Re^d\rightarrow \Re$ is twice continuously differentiable
	and Lipschitz continuous. Let $\Xt:[0,T]\rightarrow \Re^d$ be a Lipschitz continuous function, 
	and denote $g_t \eqdef G(\Xt_t)$. Then,
\begin{romannum}
\item $g:[0,T]\to \Re$ is Lipschitz continuous.
\item At any time $t_0\in(0,T)$ such that the derivatives of $g_t$ and $\Xt_t$ exist, 
\begin{equation}
	\label{e:chain1}
	\begin{aligned}
		\ddt g_t\Bigr|_{t=t_0} 
		= \partial_\theta  G_k(\Xt_{t_0})\cdot \ddt \Xt_t
		\Bigr|_{t=t_0}  \quad \mbox{for each } k\in\argmax_i G_i(\Xt_{t_0}) .
	\end{aligned}
	\end{equation}
\end{romannum}
\end{lemma}
\begin{proof}
	Denote $g_t^i = G_i(\Xt_t)$, so that $g_t=\max_{1\leq i\leq \ell_u} g_t^i$. Let $b_L$ denote a Lipschitz constant for each of these functions:
\[
	|g_{t_1}^i-g_{t_0}^i|\leq b_L|t_1-t_0|,\qquad t_0,t_1\in [0,T],\quad 1\leq i\leq \ell_u
\]
	For any $t_0,t_1\in[0,T]$,
\[
\begin{aligned}
	g_{t_1} - g_{t_0}&\leq g_{t_1}^k - g_{t_0}^k, \quad \mbox{for each } k\in\arg\max_i g_{t_0}^i\\
	&\leq b_L|t_1-t_0|
\end{aligned}
\]
The same inequality holds for $g_{t_0}-g_{t_1}$ with $k\in\arg\max_i g_{t_1}^i$. This proves (i).
	
The proof of (ii) is also straightforward: The difference $g_{t}-g_{t}^k$ has a global minimum at $t_0$ if $k\in\arg\max_i g_{t_0}^i$,  and consequently
\[
	0 = \ddt[g_t - g_{t}^k] \bigr|_{t=t_0}
\]
\end{proof}

Given a parameter vector $\theta \in \Re^d$, denote by $\condr^\theta: \state \times \U \to \Re$  the reward function that satisfies the Bellman equation \eqref{e:qfunction}, with $Q^*$ replaced by $Q^\theta$: For each $x \in \state$ and $u \in \U$,
\begin{equation}
\label{e:condtd}
\condr^{\theta} (x, u) \eqdef -\disc\sum_{x'\in \state}P_u(x,x')\uQ^\theta(x') + Q^\theta(x,u)
\end{equation}
\begin{lemma}
\label{t:condChain}
	 Suppose Assumptions A1-A2 hold and the function $\Xt:[0,T]\rightarrow \Re^d$ is Lipschitz continuous. Then, $\condr^{\Xt_t}(x,u)$ is Lipschitz continuous in $t$ for each $x,u$. Moreover, at any point $t_0$ of differentiability,
\begin{equation}
\label{e:chain2}
	\ddt \condr^{w_t}(x,u)\Bigr|_{t=t_0} =   \Big[-\disc\sum_{x'\in\state} P_u(x,x')\partial_\theta  Q^{\Xt_{t_0}}(x',\phi^{\Xt_{t_0}}(x')) + \partial_\theta  Q^{\Xt_{t_0}}(x,u) \Big]\ddt \Xt_t \Bigr|_{t=t_0}
\end{equation}
where $\phi^{\Xt_{t_0}}$ is defined in \eqref{e:phi}.
\end{lemma}
\begin{proof}
From the definition~\eqref{e:condtd}, it is sufficient to establish the derivative formula
\[
	\ddt \uQ^{\Xt_t}(x')\bigr|_{t=t_0}=\partial_\theta  Q^{\Xt_{t_0}}(x',\phi^{(k)}(x'))\cdot\ddt \Xt_t\bigr|_{t=t_0}
\]
	where $\phi^{(k)}$ is \textit{any} policy that is $Q^{\Xt_{t_0}}$-greedy. This is immediate from \Lemma{t:chain1}.
\end{proof}

Stability of is obtained from the following standard Lyapunov condition:
\begin{lemma}
\label{t:odeFoster}
Suppose that $\{ \Xt_t : t\in\Re\}$ is a Lipschitz continuous function taking values in $\Re^d$,   $V\colon\Re^d\to\Re_+$ is continuous and coercive, and $U\colon\Re^d\to\Re_+$.   Assume moreover the following properties:
\begin{romannum}
\item  $\inf \{ U(\theta)  :   V(\theta) \ge \delta \} >0$ for each $\delta>0$.
\item  $\displaystyle V(\Xt_t) \le 
V(\Xt_0)   -   \int_0^t  U(\Xt_\tau) \, d\tau\,,\qquad t\ge 0$.
\end{romannum}
Then, there exists a function $B\colon\Re_+\to\Re_+$ such that 
$V(\Xt_t) \le \eta$ for all $t\geq V(\Xt_0) B(\eta)$.  In particular,  $V(\Xt_t)\to 0$ as $t\to\infty$.
\end{lemma}

\begin{proof}
For any scalar $\eta$ satisfying $0 < \eta <V(\Xt_0)$, let $H^{\eta} \eqdef \{\theta: \eta \le \theta \leq V(\Xt_0)\}$   
and
\[
	\epsy_\eta= \inf_{\theta \in H^{\eta}}  U(\theta) 
\]
Under   assumption (i) of the lemma we have $\epsy_\eta > 0$. 
Let $T^\eta = \inf\{t: V(\Xt_t) \le \eta\}$,  so that $\Xt_t \in H^{\eta}$ for $0\le t \le T^\eta $.   By assumption (ii) we have
\[
0\le	V(\Xt_t) \leq V(\Xt_0) -  \epsy_\eta t\,, \qquad   0<t \le T^\eta .   
\]
Therefore, $T^\eta < V(\Xt_0)/\epsy_\eta$. Because $V(\Xt_t)$ is non-increasing in $t$, we have $V(\Xt_t) \le \eta$ for all $t\geq V(\Xt_0) B(\epsy)$, with $B(\epsy) =\epsy_\eta^{-1}$.

Since $\eta$ is arbitrary, it follows that $\lim_{t\to \infty} V(\Xt_t)=0$.   
\end{proof} 

\begin{proof}[Proof of \Prop{t:odestable}]
	Suppose $\Xt :[0,T]\rightarrow \Re^d$ is a solution of \eqref{e:inteode}. At point $t$ of differentiability, the derivative of $\barf(\Xt_t)$ is given by
\begin{equation}
\label{e:fderivative}
\begin{aligned}
	\ddt \barf(\Xt_t) =&\ddt \Expect\Big[\elig_n \condr^{\Xt_t}(X_n,U_n) \Big]\\
	=&\Expect[\elig_n \ddt \condr^{\Xt_t}(X_n, U_n)] +\Expect[ \clD(\Xt_t, \Phi_{n+1})\ddt \elig_n ]
\end{aligned}
\end{equation}
For each $x\in \state$ and $u\in\action$, $\condr^{\Xt_t}(x, u)$ is a Lipschitz continuous function of $t$, whose derivative is given in \Lemma{t:condChain}. 
Assertion~(i) follows:
\[
\begin{aligned}
	\ddt \barf(\Xt_t) &= \Expect\Big[\elig_n [\disc \partial_\theta  Q^{\Xt_t}(X_{n+1},\phi^{\Xt_t}(X_{n+1})) - \partial_\theta  Q^{\Xt_t}(X_n,U_n)] + \clD(\Xt_t, \Phi_{n+1})\partial_\theta  \elig_n\Big]\ddt \Xt_t \\
	& = -A(\Xt_t)[\epsy I + A(\Xt_t)^\transpose A(\Xt_t)]^{-1}A(\Xt_t)^\transpose \barf(\Xt_t)
\end{aligned}
\]

A candidate Lyapunov function is defined as $V(\Xt_t) \eqdef \frac{1}{2}\|\barf(\Xt_t) \|^2$.  At a point $t$ where $\barf(\Xt_t)$ is differentiable,
\begin{equation}
\label{e:lyp-zapQ}
	\ddt V(\Xt_t) = -\barf(\Xt_t)^\transpose A(\Xt_t)[\epsy I + A(\Xt_t)^\transpose A(\Xt_t)]^{-1}A(\Xt_t)^\transpose \barf(\Xt_t)	
\end{equation}
\[
	-[\epsy I + A(\theta)^\transpose A(\theta)]^{-1} \leq - b_V I
\]
The integral representation of \eqref{e:lyp-zapQ} then gives, for any $t\in [0, T]$,
\begin{equation}
\label{e:lyp-zapQ-int}
\begin{aligned}
	V(\Xt_t) &= V(\Xt_0) - \int_0^t\barf(\Xt_\tau)^\transpose A(\Xt_\tau)[\epsy I + A(\Xt_\tau)^\transpose A(\Xt_\tau)]^{-1}A(\Xt_\tau)^\transpose \barf(\Xt_\tau)\, d\tau 	 \\
				&\leq V(\Xt_0) - b_V \int_0^t \|A(\Xt_\tau)^\transpose \barf(\Xt_\tau)\|^2 \, d\tau
\end{aligned}
\end{equation}
Under (A2) the assumptions of \Lemma{t:odeFoster} hold with $U(\theta) = b_V \|A(\theta)^\transpose \barf(\theta)\|^2$, so that $\lim_{t\to\infty } V(\Xt_t)  = \lim_{t\to\infty } \barf(\Xt_t)  = 0$. 

If in addition (A3) holds, we conclude that 
   $\lim_{t\to \infty}\Xt_t = \theta^*$.
\end{proof}

\subsection{Proof of \Theorem{t:ZapQ-main}}
\label{s:proof-ode-approx}

The remainder of the Appendix is dedicated to the proof of \Theorem{t:ZapQ-main}.    We use $n_0=0$ in the definition of the step-size sequences \eqref{e:gains};  this shortens many of the expressions that follow, and the extension to general $n_0\ge 1$ is obvious. Given the typical choice of $\zeta_n$ in \eqref{e:elig}, it is assumed throughout that $\elig_n \eqdef \zeta(\theta_n,X_n, U_n)$ for some function $\zeta:\Re^d\times \state\times \action\to \Re^d$. We proceed under the additional assumption that the vector-valued function $\zeta$ has non-negative entries:
\begin{equation}
\textit{
 $[\elig(\theta, x,u)]_i\geq 0$ for each  $i, \theta, x,u$.  
}
\label{e:AN}
\end{equation} 
The proofs are extended to the general case in \Section{s:general-zeta}.

\subsubsection{Generalities}
\label{s:generalities}
This subsection contains the building blocks of the proof, summarized in two propositions, and the proof of \Theorem{t:ZapQ-main} based on these key results.  The proofs of the propositions are postponed to subsequent subsections.  
 
The \textit{slow time scale} used for an ODE approximation of $\{\theta_n\}$ is defined by
\begin{equation}
t_n = \sum_{i=1}^n \alpha_i=\sum_{i=1}^n \frac{1}{i}\,, \qquad n \geq 1\,, \qquad t_0 = 0
\label{e:slow-time}
\end{equation}
and its approximate inverse
\begin{equation}
\label{e:slow-time-inv}
[t] \eqdef \max\{j: t_j \leq t \}
\end{equation}
Define the continuous time process $\{\barxt_{t}: t\geq 0\}$ with $\barxt_{t_n} = \theta_n$, and extended to $\posRe$ via linear interpolation. Define the associated continuous time process $\{\barc_t \eqdef \barf(\barxt_{t}): t\geq 0\}$.  We also define the piecewise constant processes $\{\barclA_t, \barclG_t: t\geq 0 \}$ with $\barclA_{t} = \haA_{n+1}, \barclG_t = G_{n+1}$ for $t\in [t_n, t_{n+1})$.     Both $b_\theta \eqdef \sup_n \|\theta_n\|= \sup_{t} \|\barxt_t\|$ and $b_c \eqdef \sup_t \|\barc_t\|$ are finite \emph{a.s.} by assumption.    

Denote by $\so(1) = e(T_0, t )$   a function of two variables, satisfying for each $T>0$,
\[
\lim_{T_0\to \infty} \sup_{0\leq t \leq T} \|e(T_0, t )\| = 0
\]
\begin{proposition}
\label{t:Gamma12}
	Under Assumptions (A1)-(A2) and \eqref{e:AN},    $\{\barxt_t\}$ and  $\{\barc_t\}$ are Lipschitz continuous with respect to $t$, and the following approximations hold:
\begin{romannum}
\item  $\displaystyle
		\lim_{T_0\to \infty}\int_{T_0}^{T_0+T}\|\barclA_t \ddt \barxt_t - \ddt \barc_t \|_\infty \, dt =  0
		$.
		
\item $\displaystyle
		\barc_{T_0+ t}= \barc_{T_0} + \int_{T_0}^{T_0+t} \barclA_{\tau} \barclG_\tau\barc_\tau \, d\tau + \so(1) \,, \qquad T_0 \to \infty
		$
		
\item $\displaystyle 
		\|\barc_{T_0+ t}\|^2= \|\barc_{T_0}\|^2 + 2\int_{T_0}^{T_0+t} \barc_\tau^\transpose \barclA_{\tau} \barclG_\tau\barc_\tau \,d\tau + \so(1)\,, \qquad T_0 \to \infty
		$.
		\qed
\end{romannum}
\end{proposition}

For a fixed but arbitrary time-horizon $T>0$, define a family of functions $\{\barGamma^{T_0} :  T_0 \geq 0\}$, where $\barGamma^{T_0}: [0,T] \to \Re^m$ for each $T_0\geq 0$ and an integer $m$. It consists of four components: for $t\in [0,T]$,
 \[
 \barGamma_1^{T_0}(t) = \barxt_{T_0 +t}\,,  \qquad \barGamma_2^{T_0}(t) =\barc_{T_0 +t} \,, \qquad  \barGamma_3^{T_0}(t) = \barclA_{T_0 +t}\,, \qquad \barGamma_4^{T_0}(t) = -\barclA_{T_0 +t}\barclG_{T_0+t}
 \]
 $\{\barGamma_1^{T_0}: T_0\geq 0\}$ and $\{\barGamma_2^{T_0}: T_0\geq 0\}$ are uniformly Lipschitz continuous and bounded. More specifically, each of $\barGamma_1^{T_0}$ and $\barGamma_2^{T_0}$ is a function of two variables: $\barGamma_1^{T_0}(\omega ,t), \barGamma_2^{T_0}(\omega ,t)$ with $\omega\in \Omega$ and $t\in \posRe$. The property that $\barGamma_1^{T_0}$ and $\barGamma_2^{T_0}$ are Lipschitz continuous and bounded holds with probability one. Denote their sub-sequential limits by
 \[
 \Gamma_1(t) = \Xt_t\,, \qquad \Gamma_2(t) = c_t
 \]
	where the convergence is uniform over $[0,T]$.
	
Limits of the remaining components of $\Gamma$ are defined with respect to the weak topology in $L_2([0,T]; \Re^{d\times d})$.  Because 
$\{\barGamma_3^{T_0}:T_0\geq 0\}$ and $\{\barGamma_4^{T_0}: T_0\geq 0\}$ are uniformly bounded, they are   weakly relatively sequentially compact in $L_2([0,T]; \Re^{d\times d})$~\cite[Theorem 1.1.2]{eva90}. Their weak sub-sequential limits $\Gamma_3$ and $\Gamma_4$ are denoted by $\{\clA_t, \clH_t: 0\leq t \leq T\}$. That is, there exists $T_k\to \infty$ such that
\[
\barGamma_3^{T_k} \to \clA  \text{ weakly in }  L_2([0,T]; \Re^{d\times d}) \,, \qquad  \barGamma_4^{T_k}  \to \clH \text{ weakly in }  L_2([0,T]; \Re^{d\times d})\,, \qquad k\to \infty
\]

Based on \Prop{t:Gamma12}, and a separate analysis of the \emph{fast time scale} recursion for $\{\haA_n\}$ we obtain the following properties for any sub-sequential limit $\Gamma$ of $\{\barGamma^{T_0}:  T_0 \geq 0\}$:
\begin{proposition}
\label{t:sa-sub-seq}
Under Assumptions (A1)-(A2) and \eqref{e:AN}, for each $t\in [0,T]$,
\begin{romannum}
\item $c_t\eqdef \Gamma_2(t)   = \barf(\Xt_t)$.
\item $\clA_t \eqdef \Gamma_3(t) \in \clA(\Xt_t)$.
\item $\clH_t \eqdef \Gamma_4(t) \in \Re^{d\times d}$ is positive semi-definite.
\item There exists a constant $b_V>0$ such that, for a.e. $t\in [0, T]$,
\begin{subequations}
		\begin{align}
			\ddt c_t &= -\clH_t c_t 
\label{e:ode-c-exp-pos} 
			\\
			\ddt V(\Xt_t) &\leq -  U(\Xt_t) 
\label{e:ode-c-lya-pos} 
\end{align}
\label{e:ode-c-pos}
\end{subequations}%
with 			$V(\Xt_t) = \half \|\barf(\Xt_t) \|^2 $ and $U(\Xt_t) =  b_V\|\clA_t^\transpose c_t\|^2 $.
\qed
\end{romannum}%
\end{proposition}%

An alert reader will notice that we have \textit{not} obtained the desired ODE limit,  since \eqref{e:ode-c-exp-pos} may differ from the ODE solution given in 
\Lemma{t:odestable}~(i).   In particular, we do not know if $\clA_t $ coincides with $  A(\Xt_t)$  (where $A(\theta)$ is defined in \eqref{e:barfder} using a particular $Q^\theta$-greedy policy), and we do not know if $\clH_t $ coincides with
\[
A(\Xt_t)[\epsy I + A(\Xt_t)^\transpose A(\Xt_t)]^{-1}A(\Xt_t)^\transpose
\]
We preserve the essential drift condition \eqref{e:ode-c-lya-pos},  which leads to a simple proof of the main result:

\begin{proof}[Proof of \Theorem{t:ZapQ-main}]   \Prop{t:sa-sub-seq}~(i) and (ii) justify the assertion that $U(\Xt_t) \eqdef  b_V\|\clA_t^\transpose c_t\|^2 $
is in fact a function of $\Xt_t$.  Under (A2) we see that Assumption~(i) of  \Lemma{t:odeFoster} holds,  and  \eqref{e:ode-c-lya-pos}  implies Assumption~(ii) of the lemma.   

For given $\eta>0$, we may choose $T\geq V(\Xt_0) B(\eta)$, so that $V(\Xt_T) \le \eta$ for any sub-sequential limit. It then follows that $\limsup_{n\to\infty}V(\theta_n)\leq \eta$. Since  $\eta>0$ is arbitrary,
it follows that $V(\Xt_T)\equiv 0$;  that is,  $\lim_{t\to\infty}  \barf(\theta_n) =0$ as claimed. 
\end{proof}

\subsubsection{Analysis of $\{\haA_n\}$ over the fast time scale}
\label{s:fast-time}

The goal in this subsection is to show that $\haA_n$ is close to the set $\clA(\theta_n)$ with $n$ sufficiently large. An explicit representation of $\clA(\theta)$ is given in the following: denote, for any $\theta\in\Re^d$ and any (possibly randomized) policy $\phi$, the random $d\times d$ matrix:
\[
A_{n+1}(\theta, \phi)   = \bigl[ \disc\partial_\theta  Q^\theta(X_{n+1},  \phi(X_{n+1}))- \partial_\theta  Q^{\theta}(X_n,U_n)  \bigr]   \zeta(\theta,X_n, U_n)
\]
If $\phi$ is $Q^\theta$-greedy, meaning 
\[
Q^\theta(x,  \phi (x))    = \uQ^\theta(x)\,,   \qquad x\in\state  \,,
\]
then a generalized subgradient of the function $f$   in \eqref{e:safn-f} is given by
\[
A_{n+1}(\theta, \phi)  + \clD(\theta, \Phi_{n+1})\partial_\theta   \zeta(\theta,X_n, U_n) 
\]

\begin{lemma}
\label{t:chain_rule+}
	If (A1)-(A2) hold, and if $\zeta$ is non-negative, then the set $\clA(\theta)$ defined in \eqref{e:der_barf1} admits the representation,
\[
	\clA(\theta) = \bigl\{ \Expect_{\varpi}[ A_{n+1}(\theta, \tilde\phi_{n+1}^\theta) + \clD(\theta, \Phi_{n+1})\partial_\theta   \zeta(\theta,X_n, U_n)]   
	\	  :   \
	\text{ $ \tilde\phi_{n+1}^\theta $ is $Q^\theta$-greedy} \bigr\}
\]
	where  $ \tilde\phi_{n+1}^\theta $  ranges over all $Q^\theta$-greedy randomized policies. \qed
\end{lemma}

A key implication of the non-negativity assumption \eqref{e:AN} is the following:

\begin{lemma}
\label{t:taylor-f}
	Under Assumptions (A1)-(A2) and \eqref{e:AN}, there exists  $b_T< \infty$ such that, for all $n\geq 1$ and all vectors $v\in \Re^d$, $\|v\|\leq 1$, 
\begin{equation}
\label{e:convex-n}
\begin{aligned}
	f(\theta_n+  v, \Phi_{n+1})  \geq f(\theta_n, \Phi_{n+1}) + A_{n+1}v - b_T \|v\|^2 \unit
\end{aligned}
\end{equation}
	where the inequality is component-wise, $A_{n+1}$ is defined in \eqref{e:zap-A}, and $\unit$ denotes the vector of all ones. In particular, when $Q^\theta = \psi^\transpose\theta$, we have $b_T=0$:
\begin{equation}
	f(\theta_n+ v, \Phi_{n+1}) \geq f(\theta_n, \Phi_{n+1}) + A_{n+1} v 
\label{e:convex-l}
\end{equation}
	
\end{lemma}
\begin{proof}
The proof is based on the Taylor series expansion. With $z\eqdef (x', x,u', u)$, define $g:\Re^d\times \zstate  \times \action \to \Re$ by
\begin{equation}
	g(\theta, z, u^\circ) \eqdef r(x, u) + \disc Q^{\theta}(x', u^\circ) - Q^{\theta}(x,u)
\label{e:td-nomax}
\end{equation}	
	 By (A2), $g$ admits the Taylor series expansion at each $\|\theta\| \leq b_\theta$:
\[
	g(\theta+v, z, u^\circ) = g(\theta, z, u^\circ) + \partial_\theta  g(\theta, z, u^\circ)  v + O(\|v\|^2)
\]
	Recall that $\clD(\theta, z) \eqdef g(\theta, z, \phi^\theta(x')) =  \max_{u^\circ}g(\theta, z, u^\circ)$ and the state-input space is finite,
\begin{equation}
\begin{aligned}
	\clD(\theta_n+v, \Phi_{n+1}) &=  \max_{u^\circ}g(\theta_n+v, \Phi_{n+1}, u^\circ) \\
	& =    \max_{u^\circ} g(\theta_n, \Phi_{n+1}, u^\circ) + \partial_\theta  g(\theta_n, \Phi_{n+1}, u^\circ)  v + O(\|v\|^2) \\
	&  \geq   \clD(\theta_n, \Phi_{n+1}) +  \partial_\theta  g(\theta_n, \Phi_{n+1}, \phi^{\theta_n}(X_{n+1}))   v + O(\|v\|^2) 
\end{aligned}
\label{e:taylor-g}				
\end{equation}
	Denote  $\elig_n(\theta) \eqdef \zeta(\theta, X_n, U_n)$. Another Taylor series expansion of $\zeta$ at $\theta_n$ gives 
\begin{equation}
	\elig_n(\theta_n + v) = \elig_n(\theta_n) + \partial_\theta  \elig_n(\theta_n)  v + O(\|v\|^2)
\label{e:taylor-zeta}
\end{equation}
	We next recall that $f(\theta_n, \Phi_{n+1}) = \elig_n(\theta_n) \clD(\theta_n, \Phi_{n+1})$,
\[
\begin{aligned}
	f(\theta_n+v, \Phi_{n+1}) - 	f(\theta_n, \Phi_{n+1}) 
	= &  \elig_n(\theta_n)\bigl\{\clD(\theta_n+v, \Phi_{n+1}) - \clD(\theta_n, \Phi_{n+1})\bigr\}\\
	&+ \bigl\{\elig_n(\theta_n+v)-  \elig_n(\theta_n)\bigr\}\clD(\theta_n, \Phi_{n+1}) \\
	& +\bigl\{ \elig_n(\theta_n+v) -  \elig_n(\theta_n)\bigr\} \bigl\{\clD(\theta_n+v, \Phi_{n+1}) - \clD(\theta_n, \Phi_{n+1})\bigr\} 
\end{aligned}
\]
	By \eqref{e:taylor-g} and the non-negativity assumption \eqref{e:AN},
\[
	\elig_n(\theta_n)\bigl\{\clD(\theta_n+v, \Phi_{n+1}) - \clD(\theta_n, \Phi_{n+1})\bigr\} \geq \elig_n(\theta_n) \partial_\theta  g(\theta_n, \Phi_{n+1}, \phi^{\theta_n}(X_{n+1}))   v + O(\|v\|^2)
\]
Similarly,  from \eqref{e:taylor-zeta},
\[
\begin{aligned}
	\bigl\{\elig_n(\theta_n+v)-  \elig_n(\theta_n)\bigr\}\clD(\theta_n, \Phi_{n+1})   = &\bigl\{\partial_\theta  \elig_n(\theta_n)  v + O(\|v\|^2)\bigr\} \clD(\theta_n, \Phi_{n+1})   \\
	\geq & \clD(\theta_n, \Phi_{n+1})\partial_\theta  \elig_n(\theta_n)  v + O(\|v\|^2)
\end{aligned}
\]
	By (A2) once more, both $\zeta$ and $\clD$ are  Lipschitz continuous in $\theta$,
\[
	\bigl\| \elig_n(\theta_n+v) -  \elig_n(\theta_n)\bigr\| \bigl|\clD(\theta_n+v, \Phi_{n+1}) - \clD(\theta_n, \Phi_{n+1})\bigr| = O(\|v\|^2) 
\]
	Consequently,
\[
\begin{aligned}
	f(\theta_n+v, \Phi_{n+1}) - 	&f(\theta_n, \Phi_{n+1}) \\
	 & \geq \bigl\{\elig_n(\theta_n) \partial_\theta g(\theta_n, \Phi_{n+1}, \phi^{\theta_n}(X_{n+1}))  + \clD(\theta_n, \Phi_{n+1})\partial_\theta  \elig_n(\theta_n)\bigr\} v  + O(\|v\|^2)
\end{aligned}
\]
The proof is completed by realizing that $A_{n+1}$ defined in \eqref{e:zap-A} can be expressed 
\[
	A_{n+1}=\elig_n(\theta_n) \partial_\theta g(\theta_n, \Phi_{n+1}, \phi^{\theta_n}(X_{n+1}))  + \clD(\theta_n, \Phi_{n+1})\partial_\theta  \elig_n(\theta_n)
\]
\end{proof}

Define the \emph{fast time scale}, over which the matrix gain sequence $\{\haA_n\}$ is updated,
\begin{equation}
t_n = \sum_{i=1}^n \stepf_i= \sum_{i=1}^n \frac{1}{i^\rho}\,, \qquad n \geq 1\,, \qquad t_0 = 0\,, \qquad \rho\in (0.5, 1) 
\label{e:fast-time}
\end{equation}
Define the time process $\{\barclA_t: t\geq 0 \}$ with $\barclA_{t_n} = \haA_n$ for those values $t_n$, with the definition extended to $\posRe$ via linear interpolation. Note that this definition of $\{\barclA_t: t\geq 0 \}$ is used only in this subsection to analyze $\{\haA_n\}$. For each $n\geq 1$, define the associated time block: $[t_{m(n)}, t_n)$ where $m(n) = \min\{j: t_j + \ln(n) \geq t_n \}$. Some properties of this fast time scale setting are collected in the following:
\begin{lemma}
\label{t:slow-time-block}
The follow hold:
\begin{romannum}
\item $\ln(n) - 1 <  t_n - t_{m(n)} \leq \ln(n)$.
\item There exists $N_s \geq 1$ such that for $n \geq N_s$, $m(n)+1 \geq \rho^{1/(1-\rho)} (n+1) $.
\item $\lim_{n\to\infty} \max_{m(n)\leq k \leq n}\|\theta_k - \theta_n\| = 0$.
\end{romannum}
\end{lemma}
\begin{proof}
	(i) follows directly from the definition.
	
	By \eqref{e:fast-time},
\begin{equation}
\label{e:timeblock-ineq}
\begin{aligned}
	t_n - t_{m(n)} =\sum_{i=m(n)+1}^n \frac{1}{i^{\rho}} 
	& \geq \int_{m(n)+1}^{n+1} \frac{1}{\tau^{\rho}} \, d\tau  \\
	&= (1-\rho)^{-1} [ (n+1)^{1-\rho}- (m(n)+1)^{1-\rho}]
\end{aligned}
\end{equation}
	Since $\ln (n) \geq t_n - t_{m(n)} $, we have
\[
	(1-\rho) \ln (n) \geq (n+1)^{1-\rho} - (m(n)+1)^{1-\rho}  
\]
There exits $N_s\geq 1$ such that $ (n+1)^{1-\rho} \geq \ln (n)$ for $n\geq N_s$. Hence,
\[
	(1-\rho) (n+1)^{1-\rho} \geq (n+1)^{1-\rho} - (m(n)+1)^{1-\rho}  \,, \qquad n\geq N_s
\]
	which proves (ii).
	
	By \eqref{e:timeblock-ineq},
\[
\begin{aligned}
	(1-\rho)^{-1} [(n+1)^{1-\rho} - (k+1)^{1-\rho}] & \leq (1-\rho)^{-1} [(n+1)^{1-\rho} - (m(n)+1)^{1-\rho}] \\
	& \leq \ln (n)
\end{aligned}
\]
	Multiplying each side of above inequality by $(1-\rho)(k+1)^{\rho-1}$ gives
\[
	\Bigl(\frac{n+1}{k+1}\Bigr)^{1-\rho} -1  \leq (1-\rho)(k+1)^{\rho-1}\ln (n)  \leq (1-\rho)(m(n)+1)^{\rho-1}\ln (n) 
\]
	By the inequality $\ln(1+x) \leq x$ for $x>-1$, 
\[
	(1-\rho) \ln\Bigl( \frac{n+1}{k+1}\Bigr) \leq \ln\bigl(1+ (1-\rho) (m(n)+1)^{\rho-1}\ln (n)\bigr) \leq (1-\rho) (m(n)+1)^{\rho-1}\ln (n)
\]
	Given $m(n)+1 \geq \rho^{1/(1-\rho)} (n+1) $ in (ii),
\begin{equation}
\label{e:slow-time-ratio}
	\ln \Bigl(\frac{n+1}{k+1}\Bigr) \leq \rho^{-1}\ln (n) (n+1)^{\rho-1}
\end{equation}
The parameter vector $\theta_n$ updated by  \eqref{e:zap-theta} can be expressed
\[
	\theta_n = \theta_k  + \sum_{i=k+1}^n \alpha_i G_i f(\theta_{i-1}, \Phi_i) \,,  \qquad m(n)\leq k < n
\]
	We can find a constant $b_f< \infty$ such that $\sup_n\allowbreak \|G_{n+1}  f(\theta_{n}, \Phi_{n+1})\|\leq b_f$ for almost every $\omega \in \Omega$. With $\alpha_i \equiv 1/i$,
\[
	\|\theta_n - \theta_k\| \leq b_f \sum_{i=k+1}^{n} \alpha_i \leq b_f \int_{k}^{n} \frac{1}{\tau}\, d\tau \leq  b_f\ln\Bigl(\frac{n}{k}\Bigr)
\]
	By \eqref{e:slow-time-ratio}, for $n \geq N_s$.
\begin{equation}
\label{e:theta-slow}
\begin{aligned}
	\| \theta_n - \theta_k \| & \leq b_f	\bigl| \ln\Bigl(\frac{n}{k}\Bigr) -\ln\Bigl(\frac{n+1}{k+1}\Bigr)   \bigr|  + b_f \rho^{-1}\ln (n) (n+1)^{\rho-1} \\
	& \leq b_f\bigl| \ln(1-\frac{1}{n+1}) + \ln(1+\frac{1}{k}) \bigr|  + b_f \rho^{-1}\ln (n) (n+1)^{\rho-1} \\
	& \leq b_f \frac{1}{k} + b_f \rho^{-1}\ln (n) (n+1)^{\rho-1} \\
	& \leq  b_f \frac{1}{\rho^{1/(1-\rho)} (n+1)-1} +  b_f \rho^{-1}\ln (n) (n+1)^{\rho-1}
\end{aligned}
\end{equation}
	where the last inequality holds given $k\geq m(n) \geq \rho^{1/(1-\rho)} (n+1)-1$. Therefore, $\max_{m(n)\leq k \leq n}\| \theta_k - \theta_n \|\to 0$ as $n\to \infty$.
\end{proof}


\begin{proposition}
\label{t:sum-int}
	Under Assumptions (A1)-(A2) and \eqref{e:AN},   the following hold for all $v\in \Re^d, \|v\|\leq 1$, and all $k\in\intgr$ between $m(n)$ and $n$:
\begin{romannum}
\item 
	\begin{equation}
	\label{e:barf-hat-inequality-l}
		\sum_{i=k+1}^n \stepf_i [f(\theta_{i-1}+ v, \Phi_i) - f(\theta_{i-1} , \Phi_i) + b_T\|v\|^2\unit] \geq \haA_n v - \haA_k v + \sum_{i=k+1}^n \stepf_i \haA_{i-1}  v  
	\end{equation}
\item For any $t\in [t_{m(n)}, t_n)$, 
	\begin{equation}
	\label{e:int-inequality-st}
		\barclA_{t_n}  v - \barclA_{t} v + \int_{t}^{t_n} \barclA_{\tau} v \, d\tau   \leq (t_n - t) [\barf(\theta_{n} + v) - \barf(\theta_{n})+ b_T\|v\|^2\unit]+ o(1)\,, \qquad n\to \infty
	\end{equation} 
	where $o(1)\to 0$ as $n\to\infty$, uniformly in $v$. 
\end{romannum}

\end{proposition}

\begin{proof}
	By \eqref{e:convex-n}, for each $n\geq 1$,
\[
	f(\theta_n+ v, \Phi_{n+1}) \geq f(\theta_n, \Phi_{n+1}) + A_{n+1} v - b_T\|v\|^2\unit \,, \qquad  v\in \Re^d 
\]
	Consequently,
\[
	\sum_{i=k+1}^n \stepf_i [f(\theta_{i-1}+v, \Phi_i) - f(\theta_{i-1} , \Phi_i)+ b_T\|v\|^2\unit] \geq \sum_{i=k+1}^n \stepf_i A_i v  \,, \qquad m(n) \leq k \leq n
\]
The gain matrix $\haA_{n}$ updated by  \eqref{e:zap-Ahat} can be expressed
\[
	\haA_{n} = \haA_k + \sum_{i=k+1}^n \stepf_i A_i - \sum_{i=k+1}^n \stepf_i \haA_{i-1} \,, \qquad m(n) \leq k \leq n
\]
Therefore, $\sum_{i=k+1}^n \stepf_i A_iv= \haA_{n}v - \haA_kv + \sum_{i=k+1}^n \stepf_i \haA_{i-1}v$. This proves (i).
	
	Now consider the sum $\sum_{i=k+1}^n \stepf_i f(\theta_{i-1}+v, \Phi_i)$ with $m(n) \leq k \leq n$. We first rewrite it in the suggestive form
\[
	\sum_{i=k+1}^n \stepf_i f(\theta_{i-1}+v, \Phi_i) = \sum_{i=k+1}^n \stepf_i [f(\theta_{i-1}+v, \Phi_i) - f(\theta_{n}+v, \Phi_i)] +  \sum_{i=k+1}^n \stepf_i f(\theta_{n}+v, \Phi_i)
\]
	By the  Lipschitz continuity of $f$ in $\theta$ and \Lemma{t:slow-time-block} (iii), the first sum on the right hand side goes to 0 uniformly in $k$ as $n\to \infty$. The second sum can be expressed
\[
	\sum_{i=k+1}^n \stepf_i f(\theta_{n}+v, \Phi_i) = (t_n - t_k) \barf(\theta_{n}+ v) + \sum_{i=k+1}^n \stepf_i [f(\theta_{n}+v, \Phi_i) - \barf(\theta_{n} + v)]
\]
	Each term in the sum on the right side has zero-mean under the stationary pmf of $(\bfmX, \bfmU)$. It goes to zero  \emph{a.s.}  for $m(n)\leq k \leq n$ as $n \rightarrow \infty$ \cite[Part II, Section 1.4.6, Proposition 7]{benmetpri90}. We then obtain 
\begin{equation}
\label{e:barf-sum-l}
	\max_{m(n)\leq k \leq n} \bigl\| (t_n - t_k)\barf(\theta_{n}+v)  -	\sum_{i=k+1}^n \stepf_i f(\theta_{i-1}+v, \Phi_i) \bigr\|  =o(1) \,, \qquad n\to \infty
\end{equation}
	Since the process $\{ \barclA_t: t\geq 0\}$ is linearly interpolated between discrete values,
\[
	\int_{t_k}^{t_n} \barclA_\tau  v \,d\tau =\half \sum_{i = k+1}^n \stepf_i [\haA_{i} + \haA_{i-1} ] v = \sum_{i=k+1}^n \stepf_i \haA_{i-1}  v + \half\sum_{i = k+1}^n \stepf_i [\haA_{i} - \haA_{i-1} ] v
\]
	where the second sum on the right hand side can be rewritten as 
\[
	\sum_{i = k+1}^n \stepf_i [\haA_{i} - \haA_{i-1} ] v = -\stepf_{k+1} \haA_kv + \stepf_n\haA_{n+1}v + \sum_{i=k+1}^{n-1} [\stepf_i - \stepf_{i+1}] \haA_i  v
\]
	which goes to zero as $n\rightarrow \infty$ given $ \sup_n \|\haA_n\| <\infty$   and $\stepf_i - \stepf_{i+1} \approx \rho i^{-1}\stepf_i$. 
	Therefore,
\begin{equation}
\label{e:haA-sum-l}
	\max_{m(n)\leq k \leq n} \bigl\| \sum_{i=k+1}^n \stepf_i \haA_{i-1}  v  - \int_{t_k}^{t_n} \barclA_\tau v \, d\tau \bigr\| =  o(1) \,, \qquad n\to \infty
\end{equation}
	Combining (i) with \eqref{e:barf-sum-l} and \eqref{e:haA-sum-l} gives, for $t\in\{t_k: m(n)\leq k\leq n\}$,
\begin{equation}
\label{e:sum-int-inequality-l}
	\barclA_{t_n} v - \barclA_{t}  v + \int_{t}^{t_n} \barclA_{\tau} v \, d\tau   \leq (t_n -t)[\barf(\theta_{n}+ v) - \barf(\theta_{n})+ b_T\|v\|^2\unit] + o(1)
\end{equation}
	For any $t\in[t_{m(n)}, t_n)$, denote $k=\max\{j: t_j \leq t\}$. 
	Letting $\delta = (t-t_k)/(t_{k+1} - t_k)$, we have
\[
	\barclA_t  v = (1-\delta ) \barclA_{t_k} v + \delta  \barclA_{t_{k+1}}  v
\]
Then,
\[
\begin{aligned}
	(1-\delta )\bigl\{ \barclA_{t_{n}}v  - \barclA_{t_{k}}  v + \int_{t_k}^{t_n} \barclA_{\tau} v \, d\tau \bigr\}& \leq
	(1-\delta )(t_n -t_k) [\barf(\theta_{n}+v) - \barf(\theta_{n})+ b_T\|v\|^2\unit]+ o(1) \\
	\delta \bigl\{\barclA_{t_{n}}v   -  \barclA_{t_{k+1}} v + \int_{t_{k+1}}^{t_n} \barclA_{\tau} v \, d\tau\bigr\} &  \leq
	\delta (t_n -t_{k+1}) [\barf(\theta_{n} +v) - \barf(\theta_{n})+ b_T\|v\|^2\unit]+ o(1)
\end{aligned}
\]
	Combining above two inequalities gives
\[
	\barclA_{t_{n}}v  - \barclA_{t}  v + \int_{t}^{t_n} \barclA_{\tau} v \, d\tau  \leq
	(t_n -t) [\barf(\theta_{n}+ v) - \barf(\theta_{n})+ b_T\|v\|^2\unit] + o(1)
\]
\end{proof}

Recall the constant $b_T>0$ introduced in \Lemma{t:taylor-f}.
For fixed matrix $\haA\in \Re^{d\times d}$ and vector $\theta \in \Re^d$, define the function $\distn:\Re^{d\times d}\times \Re^d \to \Re$ by
\begin{equation}
\label{e:dist-A-n}
\distn(\haA, \theta) = \sup_{\|v\| \leq 1} \Bigl\{ \max_i \bigl[\haA v - (\barf(\theta+v) - \barf(\theta))\bigr]_i - b_T \|v\|^2 \Bigr\}
\end{equation}
This measures how well $\haA v$ approximates the directional derivative $f'(\theta;v)$ for $v$ in the unit ball.
It is non-negative since  $v=0$ is feasible in the supremum in 
\eqref{e:dist-A-n}.  It is also continuous in both arguments:
\begin{proposition}
\label{t:dist-A-n}
	Under Assumptions (A1)-(A2) and \eqref{e:AN}, the function $\distn$ defined in \eqref{e:dist-A-n} satisfies:
\begin{romannum}
\item For fixed $\haA$ and $\theta$, the supremum in  \eqref{e:dist-A-n} is achieved.
\item $\distn(\haA, \theta)$ is non-negative and Lipschitz continuous in both $\haA$ and $\theta$.
\item If $\dist_{\clN}(\haA,  \theta) = 0$, then the following hold:   $\haA \in \clA(\theta)$,   and
\[
\text{
 If $\barf'(\theta; v) = -\barf'(\theta; -v)$   for some $\|v\|\leq 1$,   then $\haA v=\barf'(\theta;v)$.}
 \]
\end{romannum}
\end{proposition}
\begin{proof}
	With fixed $\haA$ and $\theta$, $\max_i [\haA v - (\barf(\theta+ v) - \barf(\theta))]_i$ is Lipschitz continuous with respect to $v$ by \Lemma{t:chain1} (i). Since the set $\{v: \|v\|\leq 1\}$ is compact, the supremum is achieved.
	
For (ii),  consider $\haA\neq \haA'$, while $\theta$ is fixed. Let $v^*, i^*$ maximize $[\haA v - (\barf(\theta+ v) - \barf(\theta))]_i- b_T \|v\|^2$. We have
\[
\begin{aligned}
	\distn(\haA, \theta) - \distn(\haA', \theta) 
	&\leq [\haA v^* - (\barf(\theta+v^*) - \barf(\theta))]_{i^*} - [\haA' v^* - (\barf(\theta+v^*) - \barf(\theta))]_{i^*} \\
	& \leq \|\haA - \haA'\|_1\|v^*\|_1
\end{aligned}
\]
Therefore, $\distn(\haA, \theta)$ is Lipschitz continuous in $\haA$. The same argument implies the Lipschitz continuity of $\distn(\haA, \theta)$ in $\theta$.
	
	For (iii), the first claim follows from the definition of $\clA(\theta)$ in  \eqref{e:der_barf1}. By the definition of directional derivative,
\begin{equation}
\label{e:dire-deriv}
	\barf'(\theta; v) =  \barf(\theta+ v) - \barf(\theta)  + o(\|v\|) 
\end{equation} 
where $o(s)/s \to 0$ as $s \downarrow 0$. Given $\distn(\haA, \theta) = 0$, we have for each $v\in \Re^d$,
\[
\begin{aligned}
	\haA v \leq \barf(\theta+ v) - \barf(\theta) + b_T\|v\|^2\unit &= \barf'(\theta; v) + o(\|v\|)  \\
	-\haA v \leq \barf(\theta- v) - \barf(\theta) + b_T\|v\|^2\unit&= \barf(\theta; -v) + o(\|v\|) 
\end{aligned}
\]
	Using $\barf'(\theta; -v) = -\barf'(\theta; v)$ gives
\[
	\barf'(\theta; v) - o(\|v\|) \leq  \haA v \leq \barf'(\theta; v) + o(\|v\|)
\] 
	With $\barf'(\theta; s v)/s = \barf'(\theta; v)$ for $s>0$, replace $v$ by $sv$ in the above inequality and divide:
\[
	\barf'(\theta; v) - \frac{o(s\|v\|)}{s} \leq  \haA v \leq \barf'(\theta; v) + \frac{o (s \|v\|)}{s}
\]
	Letting $s \downarrow 0$ gives $\haA v = \barf(\theta; v)$.
\end{proof}

\begin{proposition}
\label{t:At-bound}	Under Assumptions (A1)-(A2) and \eqref{e:AN},
\begin{romannum}
\item The component-wise inequality holds:
	\begin{equation}
	\label{e:At-bound}
		\haA_{n} v \leq \barf(\theta_n + v) - \barf(\theta_n) + b_T\|v\|^2\unit +  o(1) \,, \qquad n\to \infty
	\end{equation}	
		where $o(1)\to 0$ as $n\to \infty$ uniformly in $\|v\|\leq 1$.
\item  
$\displaystyle
	\lim_{n \to \infty}\distn(\haA_n,  \theta_n) =0 $ a.s..

\item Let $\{\theta_{n_k}\}$ be a subsequence of $\{\theta_n\}$ that converges to some $\theta^{\circ}\in \Re^d$ a.s.. Then,
	\begin{equation}
	\label{e:A-subg-theta}
		\lim_{k\to \infty} \dist(\haA_{n_k}, \clA(\theta^{\circ})) = 0 \,, \qquad a.s.
	\end{equation}
		where $\dist(\haA_{n_k}, \clA(\theta^{\circ}))$ denotes the Euclidean distance between $\haA_{n_k}$ and the set $\clA(\theta^{\circ})$. 
\end{romannum}
\end{proposition}
\begin{proof}
	For fixed $n$ and $v\in \Re^d$, let $\clU: [t_{m(n)}, t_n] \rightarrow \Re^d$ denote the solution of the following linear integral equation
\begin{equation}
\label{e:linear-integral}
	\clU_t = \clU_{t_{m(n)}} - \int_{t_{m(n)}}^t \clU_\tau \, d\tau + (t- t_{m(n)})[\barf(\theta_n+v) - \barf(\theta_n) + b_T \|v\|^2]  \,, \qquad \clU_{t_{m(n)}} = \barclA_{t_{m(n)}} v
\end{equation}
	
	With $n$ fixed, $\delta_n \triangleq \max_i \bigl|[o(1)]_i\bigr|$ in \eqref{e:int-inequality-st} can be viewed as a positive constant. We claim that $\barclA_{t_n} v \leq \clU_{t_n} + \unit \delta_n$. Suppose the claim is not true. Then $[\barclA_{t_n} v]_i > [\clU_{t_n}]_i +\delta_n$ for some index $i$ between $1$ and $d$. Because $\barclA_{t} v$ and $\clU_{t}$ are both continuous functions over $[t_{m(n)}, t_n]$ and $\barclA_{t_{m(n)}}v=\clU_{t_{m(n)}}$, there  exists $t \in [t_{m(n)}, t_n)$ such that $[\barclA_{t} \, v]_i = [\clU_{t}]_i$ and $[\barclA_{\tau} v]_i >  [\clU_{\tau}]_i$ for $\tau\in (t, t_n)$. Consequently, combing  \eqref{e:int-inequality-st} and \eqref{e:linear-integral}  gives
\[
	\delta_n < [\barclA_{t_n} v - \clU_{t_n}]_i
	\leq [\barclA_{t} \, v - \clU_{t}]_i - \int_{t}^{t_n}[\barclA_{\tau}v - \clU_{\tau}]_i \, d\tau  + \delta_n  <\delta_n
\]
	which is a contradiction. Therefore, $\barclA_{t_n} v \leq \clU_{t_n} +\unit \delta_n$.
	
The integral equation \eqref{e:linear-integral} has the solution,
\[
	\clU_{t} = \exp( t_{m(n)}-t)\clU_{t_{m(n)}} + (1 - \exp(t_{m(n)}-t))[\barf (\theta_n + v) - \barf(\theta_n) + b_T \|v\|^2] \,, \qquad t\in[t_{m(n)}, t_n]
\]
	Consequently,
\[
\begin{aligned}
		\haA_{n} v \leq &\barf (\theta_n+ v) - \barf(\theta_n) + b_T\|v\|^2\unit +  \delta_n\unit \\
		&+ \exp( t_{m(n)}-t_n) \bigl[ \clU_{t_{m(n)}} -(\barf (\theta_n + v) - \barf(\theta_n) + b_T \|v\|^2) \bigr]
\end{aligned}
\]
	By \Lemma{t:slow-time-block} (i), we have $t_{m(n)} -t_n < -\ln(n) + 1$ and hence $\exp( t_{m(n)} - t_n ) < e/n$. Therefore,
\[
\begin{aligned}
		\bigl\| \exp( t_{m(n)}-t_n) \bigl[ \clU_{t_{m(n)}} -[\barf (\theta_n + v) - \barf(\theta_n)+ b_T\|v\|^2\unit] \bigr\| & \leq \frac{e}{n}[b_{\clA}+ b_L+b_T]\|v\| \\
			& \leq \frac{e}{n}[b_{\clA}+ b_L + b_T]
\end{aligned}
\]
which goes to zero as $n\to \infty$. This proves (i), and (ii) follows by the definition of $\distn$.
	
	We prove (iii) by contradiction: Suppose \eqref{e:A-subg-theta} does not hold. Then there exists a constant $\delta > 0$ and a subsequence $\{\haA_{n_k}\}$ such that $\dist(\haA_{n_k}, \clA(\theta^\circ))\geq \delta$ for each $k$. Without loss of generality, the subsequence is convergent, with limit $\haA^\circ$ satisfying $\dist(\haA^\circ, \clA(\theta^\circ)) \geq \delta$. However, combining statement (i) and \Proposition{t:dist-A-n} (iii) gives
\[
	\dist(\haA^{\circ},  \clA(\theta^{\circ})) = 0
\]
	which is a contradiction.
\end{proof}

\subsubsection{Proofs of \Prop{t:Gamma12} and \Prop{t:sa-sub-seq}}

In this subsection, the time processes involved all refer to those defined in \Section{s:generalities} with respect to the slow time scale \eqref{e:slow-time}.
\begin{proof}[Proof of \Prop{t:Gamma12}]
The Lipschitz continuity of $\{\barxt_t\}$ and  $\{\barc_t\}$ follows directly from  boundedness of $\{\theta_n\}$.
	
	At a point of differentiability, let $v_t=\ddt \barxt_t = \clG_{t} f(\theta_{[t]}, \Phi_{[t]+1})$ and recall that $\sup_{t} \|v_t\|\leq b_f$. Whenever exists, the derivative of $\barc_t$ may be represented as the directional derivative of $\barf(\barxt_t)$ along direction $v_t$:
\begin{equation}
\label{e:deriv-dire-deriv}
\begin{aligned}
	\ddt \barc_t = \lim_{s \to 0} \frac{\barf(\barxt_{t+ s}) - \barf(\barxt_t)}{s} &= \lim_{s \downarrow 0} \frac{\barf(\barxt_{t+ s}) - \barf(\barxt_t)}{s} = \barf'(\barxt_t; v_t)  \\
	& =\lim_{s \uparrow 0} \frac{\barf(\barxt_{t+ s}) - \barf(\barxt_t)}{s}  =  -\barf'(\barxt_t; -v_t) 
\end{aligned}
\end{equation}
	\Proposition{t:dist-A-n} (ii) combined with \Prop{t:At-bound}~(ii) gives
\begin{equation}
	\lim_{t\to \infty} \distn(\barclA_t, \barxt_t) \leq  0 \,, \qquad a.s.
\label{e:dist-xt-at}
\end{equation}
	Let $\eta_t \eqdef \max(1/t, \distn(\barclA_t, \barxt_t))$, satisfying $\eta_t >0$ and $\eta_t \to 0$ as $t\to \infty$. There exists $T_\bullet < \infty$ \emph{a.s.} such for $t\geq T_\bullet$,
\begin{equation}
\label{e:A-direction}
\begin{aligned}
	\barclA_t v_t - \barf'(\barxt; v_t) &= \frac{1}{\sqrt{\eta_t}} \bigl[\barclA_t \sqrt{\eta_t}v_t - \barf'(\barxt_t; \sqrt{\eta_t}v_t)
	\bigr] \\
	&= \frac{1}{\sqrt{\eta_t}} \bigl[\barclA_t \sqrt{\eta_t}v_t - [\barf(\barxt_t + \sqrt{\eta_t}v_t) - \barf(\barxt_t)]\bigr] + o(\|v_t\|)
	\\
	& \leq (1+b_Tb_f)\sqrt{\eta_t}\unit  + o(\|v_t\|)
\end{aligned}
\end{equation}
	where the second equality follows from \eqref{e:deriv-dire-deriv} and the last inequality holds given $\distn(\barclA_t, \barxt_t)\leq \eta_t$ and $\|v_t\|$ is uniformly bounded by $b_f$. 
	
	At points of differentiability, we apply $\barf'(\barxt_t; v_t) =-\barf'(\barxt_t; -v_t)$ from \eqref{e:deriv-dire-deriv}:
\[
	-\barclA_t v_t + \barf'(\barxt; v_t) \leq  (1+b_Tb_f)\sqrt{\eta_t}\unit  + o(\|v_t\|)
\]
	Consequently,
\[
	\| \barclA_t \ddt \barxt_t - \ddt \barc_t \|_{\infty} \leq  (1+b_Tb_f)\sqrt{\eta_t}  + o(\|v_t\|)
\]
where $\|\cdot\|_{\infty}$ denotes the infinity norm.  
The right hand side of above inequality is bounded and converges to zero as $t\to \infty$. Since the derivatives of $\barxt_t$ and $\barc_t$ exist \emph{a.e.}, we have for each $T>0$,
\[
	\int_{T_0}^{T_0+T}\| \barclA_t \ddt \barxt_t - \ddt \barc_t \|_{\infty} \, dt \leq \int_{T_0}^{T_0+T} (1+b_Tb_f)\sqrt{\eta_t}  + o(\|v_t\|) \, dt
\]
The desired result follows from Dominated Convergence Theorem.
	

Part (ii) is obtained from (i):
\[
\begin{aligned}
	\barc_{T_0 + t} & = \barc_{T_0} + \int_{T_0}^{T_0+t} \frac{d}{d\tau}\barc_\tau \,d\tau \\
	& = \barc_{T_0} + \int_{T_0}^{T_0+t} \barclA_\tau\barclG_\tau f(\theta_{[\tau]}, \Phi_{[\tau]+1})\,d\tau + \so(1) \,, \qquad T_0 \to \infty                                      \\
	& = \barc_{T_0} + \int_{T_0}^{T_0+t} \barclA_{\tau} \barclG_\tau\barf(\barxt_{\tau})\,d\tau + \so(1)\,, \qquad T_0 \to \infty
\end{aligned}			
\]
	where the last equality follows from standard ODE arguments for stochastic approximation \cite{benmetpri90}.
	
	For (iii), $\|\barc_t\|^2$ is Lipschitz continuous in $t$ given boundedness of $\{\theta_n\}$. Hence by the same argument in (ii),
\[
\begin{aligned}
	\|\barc_{T_0+ t}\|^2 & = \|\barc_{T_0}\|^2 + 2\int_{T_0}^{T_0+t} \barc_\tau^{\transpose} \frac{d}{d\tau}\barc_\tau \,d\tau  \\
	& = \|\barc_{T_0}\|^2 + 2\int_{T_0}^{T_0+t} \barc_\tau^\transpose \barclA_{\tau} \barclG_\tau\barc_\tau \,d\tau + \so(1)\,, \qquad T_0 \to \infty
\end{aligned}
\]
\end{proof}

\begin{proof}[Proof of  \Prop{t:sa-sub-seq}]
(i) follows from the Lipschitz continuity of $\barf$. 

Let $\{T_k\}$ be a sequence such that $\barGamma^{T_k} \to \Gamma$ for each of the four components: $\barGamma_i^{T_k}$, $1\leq i \leq 4$. Since $ \barGamma_3^{T_k} \to\clA$  weakly in $L_2([0,T];\Re^{d\times d})$ as $k\to \infty$,
	by the Banach-Saks theorem, there exists a subsequence $\{T_{n_k}\}$ such that
\[
	\frac{1}{N}\sum_{k=1}^{N} \barGamma_3^{T_{n_k}}(t) \to \clA_t \,, \qquad a.e. \, t\in  [0,T]\,, \qquad N\to \infty
\]
Without loss of generality,  we can modifying $\clA_t$ on a Lebesgue-null set such that the convergence above is pointwise.    
We also have $\barGamma_1^{T_{n_k}}(t) \to \Xt_t$ as $k\to\infty$ for each $t\in [0,T]$. By \Proposition{t:At-bound} (ii), 
\[
	\lim_{k\to \infty} \dist(\barGamma_3^{T_{n_k}}(t), \clA(\Xt_t)) = 0 \,, \qquad t\in [0,T]
\]
	It  follows from definition \eqref{e:der_barf1} that  the set $\clA(\theta)$ is convex for each $\theta$. Then,
\[
	\lim_{N\to \infty} \dist(\frac{1}{N}\sum_{k=1}^{N}\barGamma_3^{T_{n_k}}(t),\clA(\Xt_t) ) = 0 \,, \qquad t\in [0,T]
\]
Therefore, $\clA_t \in \clA(\Xt_t)$ for each $t\in [0,T]$. This proves (ii).
	
	Given that $\barGamma_4^{T_0}$ is positive semi-definite pointwise and uniformly bounded, the same arguments establish (iii).
	
	Since $\Gamma_2^{T_k}\to c$ uniformly over $[0,T]$ and $\Gamma_4^{T_k}\to \clH$ weakly, $\barGamma_4^{T_k}\barGamma_2^{T_k}$ converges to $\clH c:[0,T]\to \Re^d $ weakly. The ODE \eqref{e:ode-c-exp-pos} follows from \Proposition{t:Gamma12} (ii). For \eqref{e:ode-c-lya-pos}, since $b_\lambda \eqdef \sup_n \lambda_{\max} (\haA_n^\transpose\haA_n)$ is finite,
\[
	-[\epsy I + \haA_n^\transpose \haA_n]^{-1} \leq -\frac{1}{\epsy + b_{\lambda}} I \,, \qquad n\geq 1
\]
	Combining this inequality with \Proposition{t:Gamma12} (iii) implies
\begin{equation}
\label{e:ode-sa-c2-ineq}
	\|\barc_{T_0+ t}\|^2\leq  \|\barc_{T_0}\|^2 - \frac{2}{\epsy + b_{\lambda}}\int_{T_0}^{T_0+t} \|\barclA_\tau^\transpose \barc_\tau \|^2  \, d\tau + \so(1) \,, \qquad T_0\to 0
\end{equation}
	We can show that $\{(\barGamma_3^{T_k})^\transpose \barGamma_2^{T_k}\}$ converges weakly to $\clA^\transpose c$ in $L_2([0,T]; \Re^d)$ by the sames arguments that we used to establish $\barGamma_4^{T_k}\barGamma_2^{T_k}\to \clH c$ weakly. Applying \cite[Theorem 2.2.1]{eva90}, we obtain for each $t\in [0, T]$,
\[
	\int_{0}^{t} \|\clA_\tau^\transpose c_\tau \|^2  d\tau \leq \liminf_{k\to \infty} \int_{0}^{t} \|[\barGamma_3^{T_k}(\tau)]^\transpose \barGamma_2^{T_k} (\tau)\|^2 \, d\tau 
\]
	Consequently,
\[
	\|c_{ t}\|^2\leq  \|c_{0}\|^2 - \frac{2}{\epsy + b_{\lambda}}\int_{0}^{t} \|\clA_\tau^\transpose c_\tau \|^2  d\tau 
\]
\end{proof}


\subsubsection{General eligibility vector $\bfelig$}
\label{s:general-zeta}

We finally come to the general model in which   \eqref{e:AN} is relaxed. For the sake of analysis,   the two functions $\clD , \elig$ in \eqref{e:safn-f} are assumed to be parameterized by separate parameters $\theta, \xi \in \Re^d$: $\clD(\theta, z), \elig(\xi, x,u)$.  This is only for clarifying calculations -- in the end we do impose $\theta=\xi$.
 Decompose the function $\elig:\Re^d\times \state\times \action\to \Re^d$ into its positive and negative components: $\elig = \elig^+ - \elig^-$, with $\elig^+ = \max(\elig, 0)$ and $\elig^- = \max(-\elig, 0)$. Define functions $f^+, f^-: \Re^d \times \Re^d \times \zstate\to \Re^d$ by
\[
f^+(\xi, \theta, z) = \zeta^+(\xi, x, u)\clD(\theta, z) \,, \qquad f^-(\xi, \theta, z) = \zeta^-(\xi, x, u)\clD(\theta, z)
\]
Next define functions $\barf^+, \barf^-: \Re^d \times \Re^d \to \Re^d$ by
\[
\barf^+(\xi, \theta) = \Expect_{\varpi}[f^+(\xi, \theta, \Phi_{n+1})] \,, \qquad \barf^-(\xi, \theta) = \Expect_{\varpi}[f^-(\xi, \theta, \Phi_{n+1})] 
\]
Let $\clA^+(\theta),\clA^-(\theta)$ denote the sets of generalized subgradients of $\barf^+, \barf^-$ with respect to $\theta$ based on \eqref{e:der_barf1}. Explicit representations of $\clA^+(\theta)$ and $\clA^-(\theta)$ can be obtained as in \Lemma{t:chain_rule+}.
With general eligibility vector $\zeta$, let $\clA(\theta)$ denote the set
\begin{equation}
\label{e:Ahat-sub}
\clA(\theta)\eqdef \{A^+ -A^-+ \Expect_{\varpi}[\clD(\theta, \Phi_{n+1})\partial_\xi   \elig_n(\theta)]: A^+\in \clA^+(\theta)\,, A^- \in \clA^-(\theta) \}
\end{equation}

 At each $\theta\in\Re^d$, denote
\[
\barf^{+}(\theta;v) \eqdef \lim_{s \downarrow 0}\frac{\barf^+(\theta, \theta+ sv) - \barf^+(\theta, \theta)}{s}
\]
with $ \barf^-(\theta;v) $ is defined similarly. Then the directional derivative $\barf'(\theta;v)$ can be expressed
\begin{equation}
\label{e:barf-der-decomp}
\barf'(\theta;v) =\lim_{s \downarrow 0}\frac{\barf(\theta+sv) - \barf(\theta)}{s}= \barf^+(\theta;v) - \barf^-(\theta;v) + \Expect_{\varpi}[\clD(\theta,\Phi_{n+1})\partial_\xi \elig_n]v \,, \quad \theta,v\in\Re^d
\end{equation}
Decompose $A_{n+1}$ in \eqref{e:zap-A} as $A_{n+1}=A_{n+1}^+ - A_{n+1}^- + A_{n+1}^{\zeta}$:
\[
\begin{aligned}
A_{n+1}^+ &= \elig_n^+[\disc \partial_\theta Q^{\theta}(X_{n+1},\phi^{\theta_n}(X_{n+1})) - \partial_\theta  Q^{\theta}(X_n,U_n) ]  \\
A_{n+1}^- &= \elig_n^-[\disc \partial_\theta Q^{\theta}(X_{n+1},\phi^{\theta_n}(X_{n+1})) - \partial_\theta  Q^{\theta}(X_n,U_n) ]   \\
A_{n+1}^\zeta &= \clD(\theta_n, \Phi_{n+1}) \partial_\xi  \elig_n
\end{aligned}
\]
 Accordingly, the matrix gain is decomposed: $\haA_{n+1} = \haAp_{n+1} -\haAn_{n+1}+\haAz_{n+1}$, and each component can be expressed in the recursive form:
\[
\begin{aligned}
\haAp_{n+1} &= \haAp_n + \stepf_{n+1}[A_{n+1}^{+} - \haAp_n] \\
\haAn_{n+1} &= \haAn_n + \stepf_{n+1}[A_{n+1}^{-} - \haAn_n] \\
\haAz_{n+1} &= \haAz_n + \stepf_{n+1}[A_{n+1}^{\zeta} - \haAz_n]
\end{aligned}
\]

\paragraph*{Analysis of $\{\haAp_n, \haAn_n,\haAz_n\}$ over the fast time scale:} Consider the fast time scale defined by \eqref{e:fast-time}. The conclusions in \Section{s:fast-time} hold for each of $\{\haAp_n\}$ and $\{\haAn_n\}$. While $\{ \haAz_n\}$ can be treated using standard SA arguments since $A_{n+1}^{\zeta}$ is Lipschitz continuous with respect to $\theta_n$ under (A2).   We obtain an extension of \Prop{t:At-bound}:
\begin{proposition}
\label{t:At-bound-n-zeta}
The following hold:
\begin{romannum}
	\item   As $n\to\infty$,
\begin{subequations}
	\begin{align*}
		\haAp_{n}v &\leq \barf^+(\theta_n, \theta_n + v) - \barf^+(\theta_n, \theta_n) + b_T\|v\|^2\unit+ o(1)      
		\\
		\haAn_{n}v &\leq \barf^-(\theta_n, \theta_n+v) - \barf^-(\theta_n,\theta_n) + b_T\|v\|^2\unit+ o(1)           
	\end{align*}
\end{subequations}
	where $o(1)\to 0$ as $n\to \infty$,   uniformly in $\|v\|\leq 1$.
	\item Let $\{\theta_{n_k}\}$ be a subsequence of $\{\theta_n\}$ that converges to some $\theta^{\circ}\in \Re^d$ a.s.. Then,
\[
	\lim_{k\to \infty}\dist(\haAp_{n_k}, \clA^+(\theta^\circ)) = 0 \,, \qquad \lim_{k\to \infty}\dist(\haAn_{n_k}, \clA^-(\theta^\circ)) = 0 \,, \qquad a.s.
\]
	\item  $\displaystyle 	\haAz_n  =  \Expect_{\varpi}[\clD(\theta_n, \Phi_{n+1})\partial_\xi  \elig_n]+ \so(1) $.
\end{romannum}
\end{proposition}

\paragraph*{Analysis of $\{\theta_n\}$ over the slow time scale:}
Going back to the slow time scale defined by \eqref{e:slow-time}, define the continuous time processes $\{\barxt_t, \barc_t: t\geq 0\}$ as before. Define similarly the piecewise constant time processes $\{\barclA_t, \barclG_t: t\geq 0\}$ as well as the three components $\{\barclAp_t, \barclAn_t, \barclAz_t: t\geq 0 \}$. 
\begin{proposition}
\label{t:direction-general}
 The conclusions  of \Prop{t:Gamma12} and \Prop{t:sa-sub-seq} hold for general eligibility vectors,  subject to the modified definition of $\clA(\theta)$ in \eqref{e:Ahat-sub}.  
\end{proposition}

\begin{proof}
For the three claims of \Prop{t:Gamma12}, it suffices to prove that \Prop{t:Gamma12} (i) holds with the new definition \eqref{e:Ahat-sub} of $\clA(\theta)$. The rest of the claims then follow from (i).

At a point $t$ where both $\barxt_t$ and $\barc_t$ are differentiable, denote $v_t=\ddt \barxt_t$. Consider 
\[
\begin{aligned}
 \lim_{s\to 0} \frac{\barf^+(\barxt_t, \barxt_{t+ s})- \barf^+(\barxt_t, \barxt_t)}{s} &=  \lim_{s\to 0}  \sum_{x,u} \varpi(x,u)\zeta^+(\barxt_t,x,u)\frac{\condr^{\barxt_{t+s}}(x,u) - \condr^{\barxt_t}(x,u)}{s} 
\end{aligned}
\]
By \Lemma{t:chain1}, $\condr^{\barxt_t}(x, u)$ is differentiable for each state-action pair and \emph{a.e.} $t$, 
and hence
\[
\barf^+(\barxt_t;v_t) =-\barf^+(\barxt_t;-v_t) \,, \qquad  \text{for } \, a.e. \, t\in \posRe
\]
The same arguments imply $ \barf^-(\barxt_t;v_t) =-\barf^-(\barxt_t;-v_t)$ for \emph{a.e.}  $t\in \posRe$. Then, with \Prop{t:At-bound-n-zeta} (i), the same arguments used to establish \Prop{t:Gamma12} (i) yield those conclusions: For each $T>0$,
\[
\begin{aligned}
\lim_{T_0\to \infty}\int_{T_0}^{T_0+T}\|\barclAp_t v_t -  \barf^+(\barxt_t;v_t)\|_{\infty} \, dt &= 0 \\
\lim_{T_0\to \infty}\int_{T_0}^{T_0+T} \|\barclAn_t v_t - \barf^-(\barxt_t;v_t)\|_{\infty} \, dt &=  0
\end{aligned}
 \qquad  
\]
 It follows from \eqref{e:barf-der-decomp} that 
\[
\ddt \barc_t =   \barf^+(\barxt_t;v_t) - \barf^-(\barxt_t;v_t) + \Expect_{\varpi}[\clD(\barxt_t,\Phi_{n+1})\partial_\xi \elig_n]v_t
\]
Therefore,
\[
\begin{aligned}
\int_{T_0}^{T_0+T} \|\barclA_t v_t - \ddt\barc_t\|_{\infty} \, dt \leq & \int_{T_0}^{T_0+T}\|\barclAp_t v_t - \barf^+(\barxt_t;v_t)\|_{\infty}  + \|\barclAn_t v_t -\barf^-(\barxt_t;v_t)\|_{\infty} \, dt \\
																		  &+ \int_{T_0}^{T_0+T}\|\barclAz_t v_t - \Expect_{\varpi}[\clD(\barxt_t,\Phi_{n+1})\partial_\xi \elig_n]v_t\|_{\infty} \, dt
\end{aligned}
\]
where the right hand side of the above inequality goes to 0 as $T_0\to\infty$.

For the conclusions of \Prop{t:sa-sub-seq}, we only need to prove (ii) with the new $\clA(\theta)$. Let $\clA_t^+, \clA_t^-, \clA_t^{\zeta}$ denote the weak sub-sequential limits of $\{\barclAp_{T_0+t},\barclAn_{T_0+t}, \barclAz_{T_0+t}: T_0\geq 0\,, 0\leq t\leq T\}$ respectively. By \Prop{t:At-bound-n-zeta} (ii), the same arguments used for \Prop{t:sa-sub-seq} (ii) apply to each of $\clA^+_t$ and $\clA^-_t$,
\[
	\clA^+_t \in \clA^+(\Xt_t) \,, \qquad 
	\clA^-_t \in \clA^-(\Xt_t) \,, \qquad t\in[0,T]
\]
	We also have $\clA_t^{\zeta} = \Expect_{\varpi}[ \clD(\Xt_t, \Phi_{n+1})\partial_\xi \elig_n]$ from \Prop{t:At-bound-n-zeta} (iii). Therefore, $\clA_t = \clA_t^+ - \clA_t^- + \clA_t^{\zeta}$,  and $\clA_t \in  \clA(\Xt_t)$ for each $t\in [0, T]$.
\end{proof}

Following the same arguments as in \Section{s:generalities}, 
the ODE approximations and ODE limits established in \Prop{t:direction-general} imply the following extension of
\Theorem{t:ZapQ-main}:
\begin{theorem}
 The conclusions of \Theorem{t:ZapQ-main} hold,  subject to the modified definition of $\clA(\theta)$ in \eqref{e:Ahat-sub}.   
\end{theorem}

\subsection{Numerical experiments: implementation details}
\label{s:numDetails}

\paragraph{Complexity of Zap Q-learning}

For the Zap Q-learning algorithm  \eqref{e:zap},  per-iteration complexity comes from various sources:
\begin{romannum}
\item  Computation of $f(\theta_n, \Phi_{n+1})$ involves a maximum to obtain $\uQ^{\theta_n}$ in   \eqref{e:safn-d}.   
\item  The derivatives $A_{n+1}=\partial_\theta f(\theta_n, \Phi_{n+1})$ are easily computed for linear parameterization of $Q^\theta$, but require back-propagation in a neural network function approximation architecture.

\item  Computation of $G_{n+1}f(\theta_n,\Phi_{n+1})$ in \eqref{e:zap-G} and \eqref{e:zap-theta} requires    (i) multiplication of two $d\times d$ matrices, and (ii) multiplying a matrix inverse and a vector. Each of these two steps has worst case computational complexity $O(d^{3})$.
\end{romannum}
As discussed in  \Section{s:numerics},  the complexity in (iii) can be reduced by updating the gain only periodically, while continuously updating estimates of $A(\theta_n)$.  
 

The complexity bound  $O(Nd^{3}/N_d + Nd^2)$ given in \Section{s:numerics}  is based on gain updates performed only at integer multiples of $N_d$.  This bound is based on the accounting (i)---(iii) above: 
 $O(d^2)$ complexity per iteration in \eqref{e:zap-Ahat},  and $O(d^{3})$ complexity for the matrix inverse (as well as the product $\haA_{n+1}^\transpose\haA_{n+1}$ appearing in \eqref{e:zap-G}).


\paragraph{Meta-parameters in experiments}

We used $\epsy = 10^{-6}$ in \eqref{e:zap-G} for Mountain car and Acrobot, $\epsy = 10^{-4}$ for Cartpole.

For the decreasing step-size rule, we used $\rho=0.85$ and $n_0=100$ in \eqref{e:gains}.
For constant step-size experiments, we used
\[
\alpha_n \equiv \alpha\,, \qquad \stepf_n \equiv  \stepf = 100\alpha
\]
The choice of $\alpha$ itself was problem specific:    $\alpha=0.002$ for the network of size $6\times 3$ in the Mountain car example; $\alpha=0.005$ for  other experiments using constant step-size.  The average reward $\clR(\phi^{\theta_n})$ defined in \eqref{e:avr-reward}
 was estimated   by running 100 independent simulations following the policy $\phi^{\theta_n}$. The deterministic upper bound $\bar\tau$ was $200$ for Mountain car and Acrobot, and $\bar\tau=1000$ for Cartpole.

\paragraph*{Q-network}  The input space $\action$ in each of the examples is a finite set of scalars.  
Recall that the size of neural networks indicated in 	\Figure{f:examples}
  refers to the size of hidden layers, with the input to the network $(x,u)$ and the output $Q^\theta(x,u)$; hence, in the Cartpole example with $(x,u)\in\Re^5$,  the network size $30\times 24\times 16$  corresponds to $\theta\in\Re^d$, with  $d= 1341$:
  \[
  d = (5 + 1) * 30 + (30 + 1)* 24 + (24 + 1)*16 + (16 +1) = 1341
  \]
 where each${}+1$ accounts for a bias parameter.

\paragraph*{Policy}  The theory developed in this paper assumes a randomized stationary policy for exploration. In our experiments, we apply the parameter-dependent $\epsilon$-greedy exploration: At  iteration $n$,
\[
U_n = \begin{cases}
\phi^{\theta_{n}}(X_n), & \text{ with probability } 1-\epsilon \\
\text{\tt rand}, & \text{ with probability } \epsilon \\
\end{cases}
\] 
We set  $\epsilon = 0.4$ for the Mountain Car and Acrobot, and $\epsilon=0.2$ for Cartpole.

\end{document}